\documentclass{article}

\PassOptionsToPackage{compress}{natbib}

\usepackage[preprint]{neurips_2025}

\usepackage{threeparttable}
\usepackage[utf8]{inputenc} %
\usepackage[T1]{fontenc}    %
\usepackage{hyperref}       %
\usepackage{url}            %
\usepackage{booktabs}       %
\usepackage{amsfonts}       %
\usepackage{nicefrac}       %
\usepackage{microtype}      %
\usepackage{xcolor}         %
\usepackage{amsmath}
\usepackage{multirow}
\usepackage{colortbl}
\usepackage{graphicx}
\usepackage{subcaption}
\usepackage[export]{adjustbox} 
\usepackage{algorithm}
\usepackage{algpseudocode}
\usepackage{array}  
\usepackage{wrapfig}
\usepackage{amsthm}
\newtheorem{lemma}{Lemma}
\newtheorem{theorem}{Theorem}
\newtheorem{assumption}{Assumption}

\title{%
A Generative Framework for Causal Estimation via Importance-Weighted Diffusion Distillation
}

\author{%
  Xinran Song\thanks{The first two authors contributed equally. Correspondence to: \texttt{mingyuan.zhou@mccombs.utexas.edu}.}~ ,~~ Tianyu Chen$^*$,~~and Mingyuan Zhou
    \\
  The University of Texas at Austin\\
  Austin, TX 78731 \\
  \texttt{\{xinran.song,tianyuchen\}@utexas.edu,~mingyuan.zhou@mccombs.utexas.edu} \\
}

\begin{document}

\maketitle

\begin{abstract}
Estimating individualized treatment effects from observational data is a central challenge in causal inference, largely due to covariate imbalance and confounding bias from non-randomized treatment assignment. While inverse probability weighting (IPW) is a well-established solution to this problem, its integration into modern deep learning frameworks remains limited. In this work, we propose Importance-Weighted Diffusion Distillation (IWDD), a novel generative framework that combines the pretraining of diffusion models with importance-weighted score distillation to enable accurate and fast causal estimation—including potential outcome prediction and treatment effect estimation. We demonstrate how IPW can be naturally incorporated into the distillation of pretrained diffusion models, and further introduce a randomization-based adjustment that eliminates the need to compute IPW explicitly—thereby simplifying computation and, more importantly, provably reducing the variance of gradient estimates.  Empirical results show that IWDD achieves state-of-the-art out-of-sample prediction performance, with the highest win rates compared to other baselines, significantly improving causal estimation and supporting the development of individualized treatment strategies. We will release our PyTorch code for reproducibility and future research.

\end{abstract}

\vspace{-3mm}
\section{Introduction}
\vspace{-1mm}
In causal inference, the Neyman–Rubin potential outcomes (PO) framework \citep{rubin2005causal} formalizes causal effects by comparing potential outcomes under different treatments. The Fundamental Problem of Causal Inference \citep{holland1986} highlights that, for any given unit, only one of the potential outcomes can be observed—the one corresponding to the treatment actually received—while the counterfactual remains unobserved. In randomized controlled trials (RCTs), randomization ensures that treatment assignment is independent of potential outcomes, thus eliminating confounding bias. However, in most observational studies, treatment assignment is typically non-random and may depend on patient-level covariates, leading to covariate imbalance and confounding. This complicates potential outcome estimation and hinders the development of reliable individualized treatment recommendations—particularly in data-scarce settings.

Existing approaches such as inverse probability weighting (IPW) \citep{robins1994estimation} address covariate imbalance by reweighting data to approximate RCTs. %
However, IPW can be unstable due to challenges in propensity score estimation \citep{liao2022variance, ding2023coursecausalinference}, particularly when propensity scores approach 0 or 1—resulting in extreme weights and high-variance estimators. These issues are further exacerbated when applying IPW to generative models, where propensity networks could be subject to miscalibration %
and covariate representations may be poorly aligned with treatment assignment \citep{pmlr-v119-kallus20a}. 
As a result, incorporating IPW in a stable and effective manner into
generative frameworks—particularly diffusion-based models—remains an under-addressed challenge.

To address the limitations of existing causal estimation methods, we propose \emph{Importance-Weighted Diffusion Distillation} (IWDD), a novel generative framework for this task. IWDD first pretrains a covariate- and treatment-conditional diffusion model using observational data, then incorporates IPW into its distillation process. An important advantage of this two-stage procedure—diffusion pretraining followed by IPW-modulated distillation—is that pretraining allows the model to fit the in-sample distribution well, while distillation focuses on learning a conditional generator that adjusts for confounding and covariate imbalance, improving robustness for out-of-sample prediction.

We further show that the IPW-modulated distillation loss can be simplified via a randomization-based adjustment under importance reweighting, eliminating the need to explicitly compute IPW. This not only simplifies implementation but also mitigates approximation bias and numerical instability associated with propensity score estimation. More importantly, the resulting importance-weighted distillation loss is theoretically shown to reduce the variance of gradient estimates, making IWDD a stable and reliable generative approach for  causal estimation.

Another inherent benefit of IWDD is its significantly faster sampling speed compared to the pretrained conditional diffusion model. It produces samples in a single forward pass through the network, while the pretrained teacher model requires many iterative refinement steps.

Through extensive empirical studies, we demonstrate that IWDD is an effective approach for training a one-step generator for causal estimation. This establishes IWDD as not only a significantly faster alternative to conditional diffusion models pretrained on observational data, but also a more accurate method for addressing confounding and covariate imbalance inherent in causal inference settings.

We summarize our key contributions as follows:
\begin{itemize}
    \item We propose IWDD, a novel generative framework for causal estimation that pretrains a conditional diffusion model and distills it into a fast and high-performing one-step generator.
    
    \item IWDD is the first to incorporate randomized control adjustment into the distillation process, enabling effective correction for confounding and imbalanced treatment assignment.

    \item We introduce an IPW-modulated diffusion distillation objective and an improved variant that eliminates the need to explicitly estimate the propensity score.
    We provide theoretical analysis showing that this variant reduces gradient variance during distillation.

    \item Empirically, IWDD achieves state-of-the-art performance on multiple benchmark datasets for causal effect estimation, advancing the development of individualized treatment strategies.

\end{itemize}

\section{Related Work}

\textbf{CATE Estimation and PO Prediction. }
Estimating the Conditional Average Treatment Effect (CATE) has been extensively studied, with approaches broadly categorized into meta-learners, representation learning, and generative models. Meta-learners such as the S-learner and T-learner \citep{K_nzel_2019} recast CATE estimation as a supervised learning problem but are sensitive to covariate imbalance. This limitation has motivated balancing-based methods such as TARNet \citep{curth2021inductivebiasesheterogeneoustreatment,curth2021nonparametric} and CFR \citep{pmlr-v70-shalit17a}.
Generative approaches like GANITE \citep{yoon2018ganite} further model counterfactual distributions using adversarial training. More recent methods adopt doubly robust strategies, including the DR-learner \citep{kennedy2023optimaldoublyrobustestimation} and RA-learner \citep{curth2021nonparametric}, which combine nuisance component estimation with pseudo-outcome regression to enhance robustness under limited data. TEDVAE \citep{Zhang_Liu_Li_2021} employs variational autoencoders to disentangle latent confounders for treatment effect estimation. While many of these methods support PO prediction, their primary focus is CATE estimation, and they often exhibit limited accuracy when predicting individual-level outcomes.

\textbf{Diffusion Models for Causal Estimation. }
Recent work has begun to apply diffusion models to causal inference within the two major frameworks: the Structural Causal Model (SCM) framework \citep{pearl_causality_nodate} and the PO framework \citep{rosenbaum_central_1983}. While both frameworks are expressive, SCM focuses on modeling causal mechanisms through structural equations and graphs, whereas PO emphasizes treatment assignment and hypothetical interventions—making it particularly well-suited for policy evaluation and randomized experiments \citep{Pearl_2015}. In the SCM setting, diffusion models have been explored for counterfactual generation \citep{sanchez_diffusion_nodate, komanduri_causal_2024, chao_interventional_2023, shimizu_diffusion_2023} and causal discovery \citep{sanchez_diffusion_2023, mamaghan2023diffusionbasedcausalrepresentation, pmlr-v238-lorch24a, pmlr-v238-varici24a}, often relying on known or inferred causal graphs and strong structural assumptions. In contrast, diffusion models under the PO framework remain relatively underexplored. DiffPO \citep{ma2024diffpo} is among the first to use conditional diffusion models to learn potential outcome distributions given covariates and treatment assignments, but it exhibits limitations in its handling of propensity reweighting and sampling efficiency. Our work, also grounded in the PO framework, directly addresses these challenges, aiming to achieve accurate individual-level potential outcome prediction and reliable treatment effect estimation.

\textbf{Diffusion Distillation. }
Diffusion models have been developed to model complex data distributions and enable high-quality sample generation in high-dimensional spaces \citep{sohldickstein2015deepunsupervisedlearningusing, song2019generative, ho2020denoisingdiffusionprobabilisticmodels, song2020score}. Despite their impressive performance across domains, their high computational cost—stemming from the need for hundreds or even thousands of iterative refinement steps—has led to the emergence of diffusion distillation techniques that compress this process into one or a few generation steps.

A foundational strategy in diffusion distillation is to minimize a statistical divergence between the model distribution and the data distribution in the noisy space induced by forward diffusion. While distribution matching in this noisy space was pioneered by Diffusion GAN \citep{wang2022diffusion, zheng2022truncated}, it relies on noisy samples to represent the noisy distribution—unlike diffusion distillation methods that leverage pretrained diffusion models to estimate the score of the noisy distribution. A widely adopted divergence in this context is the KL divergence \citep{poole2022dreamfusiontextto3dusing2d}. Although the KL divergence itself is intractable, its gradient has a tractable form that enables alternating optimization: alternating between estimating the generator’s score and updating the generator. Methods that follow this principle include Variational Score Distillation (VSD) \citep{wang2023prolificdreamerhighfidelitydiversetextto3d}, Diff-Instruct \citep{luo2024diffinstructuniversalapproachtransferring}, Distribution Matching Distillation \citep{yin2024onestep}, and their extensions.

Score identity Distillation (SiD) \citep{zhou2024score} further advances this line of work. It does not require access to the original training data and is therefore a data-free method. By viewing the forward diffusion process through the lens of semi-implicit distributions \citep{yin2018semi, yu2023hierarchical} and leveraging associated score identities \citep{robbins1992empirical, efron2011tweedie, vincent2011connection}, SiD replaces the KL divergence with a Fisher divergence and introduces a corresponding alternating optimization procedure. The resulting distillation algorithm achieves one-step generation quality comparable to that of the original pretrained diffusion model after many denoising steps.

\begin{figure}[t]
    \centering
    \includegraphics[width=1\textwidth]{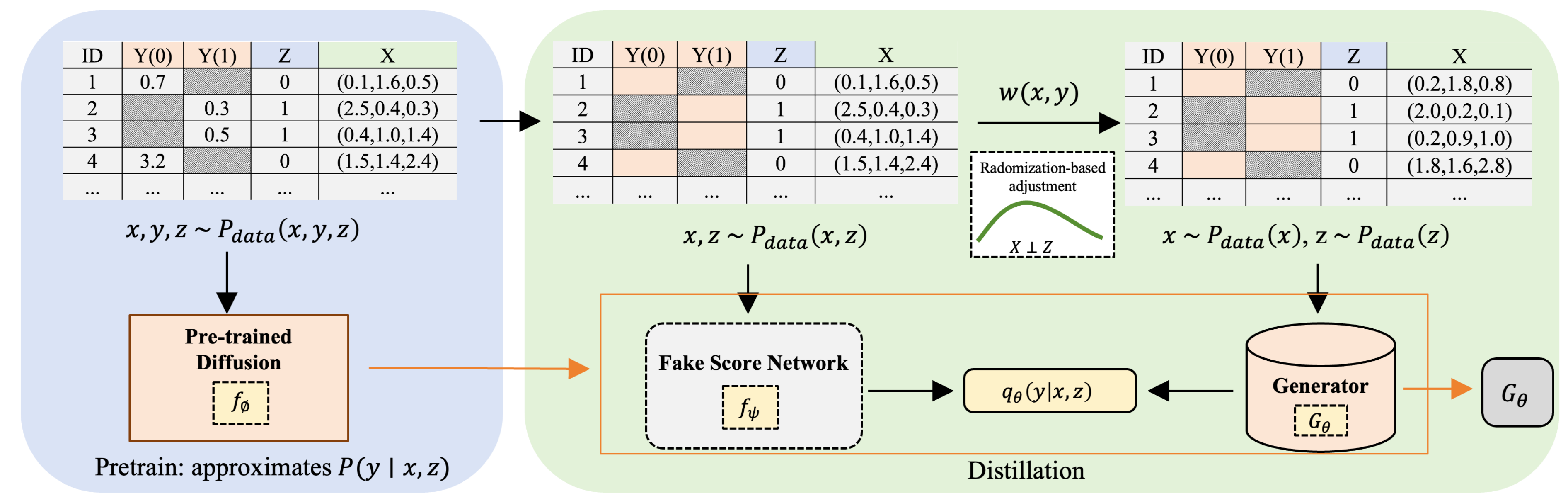}
    \caption{\small Overview of IWDD. We first pretrain a conditional diffusion model \( f_\phi(y \mid x, z) \) to approximate the true conditional distribution \( p(y \mid x, z) \) over observational data. In the distillation stage, we train a generator \( q_\theta(y \mid x, z) \) using marginal sampling, which implicitly applies %
    importance weighting without requiring explicit propensity estimation. We apply a randomization-based sampling adjustment: covariates \( x \) are shuffled and treatments \( z \) are independently sampled. The distillation algorithm is detailed in Algorithm~\ref{alg:iwdd}. }
    \label{fig:algorithm}
    \vspace{-2mm}
\end{figure}

\section{IWDD: Importance-Weighted Diffusion Distillation}
\label{sec:method}

\textbf{Notation.} Under the Neyman–Rubin PO framework \citep{rubin2005causal}, we consider an observational dataset 
\(\mathcal{D} = \bigl\{(X_i, Z_i, Y_i)\bigr\}_{i=1}^n\), 
where \(X \in \mathcal{X} \subseteq \mathbb{R}^{d}\) denotes covariates, \(Z \in \{0,1\}\) is a binary treatment indicator, and \(Y \in \mathcal{Y} \subseteq \mathbb{R}\) is the observed outcome of interest. Let \(\pi(x) = P(Z = 1 \mid X = x)\) denote the propensity score, and let \(Y(z)\) be the potential outcome under treatment \(Z = z\). 
We denote the observational data distribution as \( p_{\text{data}}(x, z, y) \), and the true conditional outcome distribution as \( p(y \mid x, z) \).
To ensure identifiability of average causal effects from observational data, we adopt the following standard assumptions:

\begin{assumption}[Consistency, Unconfoundedness, and Overlap]
\label{assump: consist}
    (1) \textbf{Consistency:} If individual \(i\) receives treatment \(Z_i\), then we only observe \(Y_i = Y_i(Z_i)\).
    (2)  \textbf{Unconfoundedness:} There are no unmeasured confounders, $i.e.$, \(\{Y_i(0), Y_i(1)\} \perp Z_i \mid X_i\).
    (3) \textbf{Overlap (Positivity):} Each individual has a non-zero probability of receiving either treatment level; that is, \(0 < \pi(x) < 1\) for all \(x \in \mathcal{X}\).
\end{assumption}
\textit{Note that under Assumption \ref{assump: consist}, we have \(\mathbb{E}[y \mid x,z] = \mathbb{E}[Y(z) \mid X = x]\), allowing us to use a fitted covariate- and treatment-conditional model to predict individual potential outcomes.}

\textbf{Problem Formulation.}
When using a generative model for causal estimation, given observational data \(\{x_i, z_i, y_i\}_{i=1}^n\), one can pretrain a covariate- and treatment-conditional diffusion model and use the reverse diffusion process to approximate \(p(y \mid x, z)\), conditioned on the input \((x, z)\). This model can then be directly applied for generative causal estimation. We show that this approach serves as a strong baseline for in-sample causal estimation. However, not only is it slow to generate random samples of \(y\) given \((x, z)\), but its performance also noticeably degrades in out-of-sample settings where \((x, z)\) lie in low-density regions of the training data.

To eliminate confounding effects and mitigate issues arising from the imbalanced distribution of treatment assignments \(z\) in the training dataset, we propose a diffusion distillation-based framework that incorporates the principles of RCTs into the distillation process. The ultimate goal of this framework is to learn a generative distribution \(q_{\theta}(y \mid x, z)\) that effectively accounts for confounding and covariate-treatment imbalance, which are common challenges in real-world observational data. We define this distribution implicitly via its generation process:
\begin{equation}
    y_g = G_\theta(x, z, \varepsilon), \quad \varepsilon \sim \mathcal{N}(0, \mathbf{I}),
\end{equation}
where \(G_{\theta}\) is a deep neural network-based one-step generator parameterized by \(\theta\). Causal estimation is then conducted by generating samples \(y_g\) from this model, conditioned on both the covariates \(x\) and the treatment assignment \(z\). 
In what follows, we detail the construction of IWDD as an effective approach for training \(G_{\theta}\), which outperforms the pretrained diffusion model in estimating the conditional distribution of \(y\) given \(x\) and \(z\), particularly when \( (x, z) \) lie in low data density regions.

\subsection{Pretraining of Covariate- and Treatment-Conditional Diffusion Models}

We begin by fitting a covariate- and treatment-conditional diffusion model \citep{sohldickstein2015deepunsupervisedlearningusing, ho2020denoisingdiffusionprobabilisticmodels,
han2022card} \( f_\phi(y \mid x, z) \), parameterized by \(\phi\), to approximate the true conditional distribution \( p(y \mid x, z) \) over observational data \( (x, z, y) \sim p_{\text{data}}(x, z, y) \):

Given a data point \( (y_0, x, z) \) from  \( p_{\text{data}}(y, x, z) \), in the forward process, Gaussian noise is gradually added to the initial outcome \( y_0 \) over \( T \) discrete time steps. This produces a sequence of progressively noisier samples \( y_1, \ldots, y_T \). The forward process is defined as: $q(y_t \mid y_0) = \mathcal{N}(a_t y_0, \sigma_t^2 I)$, with \( a_t \in [0,1] \). To generate $y_t$ given $y_0$, %
we apply the standard reparameterization: $y_t = a_t y_0 + \sigma_t \epsilon_t,  \epsilon_t \sim \mathcal{N}(0, I)$. We use the EDM schedule of \citet{karras2022edm} that sets $a_t = 1$. 
To learn the reverse process, we train the conditional denoising function \( f_\phi \) using the following objective:
\[
\mathcal{L}_\phi = \mathbb{E}_{\sigma, y, {n}}\left[\lambda(\sigma)\|f_{\phi}(y_t ; \sigma, x, z)-{y}\|_2^2\right].
\]
We follow \citet{karras2022edm}'s training schedule (details in Appendix~\ref{sec:iwdd}). A well-trained teacher diffusion model \(f_\phi\) is capable of estimating \(\mathbb{E}[y \mid y_t, x, z]\).
It serves as both the teacher and the initialization for the subsequent adjusted distillation.

DiffPO~\citep{ma2024diffpo}, a recent baseline, incorporates inverse propensity score reweighting into this diffusion loss to adjust for confounding between \(x\) and \(z\):
\begin{align}
\label{eq:loss_diffpo}
\mathbb{E}_{(y_0, x, z) \sim p(y, x, z), \; \epsilon \sim \mathcal{N}(0, I), \; t} \bigl[ w(x,z) \| \epsilon - \epsilon_\phi( y_0 + \sigma_t \epsilon, t \mid x, z) \|^2 \bigr].
\end{align}
where $w(x, z) = \frac{1}{p(z \mid x)} = \frac{z}{\pi(x)} + \frac{1 - z}{1 - \pi(x)}$, and $\pi(x) = p(z = 1 \mid x)$ denotes the propensity score.

\subsection{Distillation via Importance Reweighting}

Unlike training a diffusion model, which requires multiple reverse steps for sampling, our goal is to train a one-step conditional generator \(q_\theta(y \mid x, z)\) that approximates the true conditional distribution \(p(y \mid x, z)\) for all \((x, z) \in \mathcal{X} \times \mathcal{Z}\). In observational data, however, samples \((x,z)\sim p_{\text{data}}(x,z)\) are typically \emph{not} drawn
independently across covariates and treatment. That is, the treatment assignment~\(z\) may depend on covariates \(x\), inducing imbalance across treatment groups.
Consequently, a model that optimizes a vanilla divergence
\(\mathbb{E}_{p_{\text{data}}(x,z)}[D(q_\theta(y\mid x,z), p(y\mid x,z))]\) can bias distillation toward regions where  \((x, z)\) pair  occurs more frequently.  %
In practice, this can lead to better performance on the majority treatment group while degrading generalization in underrepresented regions, as demonstrated in our synthetic example in Section~\ref{sec:synthetic-iwdd}.

\textbf{IPW-based importance weighting.}  
To correct for the sampling bias arising from the non-random treatment assignment in the observed joint distribution \(p_{\text{data}}(x, z)\), we apply an importance weighting factor based on the discrepancy between \(p_{\text{data}}(x, z)\) and the product of marginals \(p_{\text{data}}(x) p_{\text{rct}}(z)\), where \(p_{\text{rct}}(z) = \text{Bernoulli}(z; 0.5)\) reflects the ideal joint distribution of \(x\) and \(z\) under RCTs. We reweigh every sample by:
\begin{align}
w(x, z) = \frac{p_{\text{data}}(x)p_{\text{rct}}(z)}{p_{\text{data}}(x, z)}=
\frac{p_{\text{rct}}(z)}{p_{\text{data}}(z \mid x)},\label{eq:w}
\end{align}
where \(p_{\text{data}}(z=1\mid x)\equiv\pi(x)\) is the \emph{propensity score}
\citep{rosenbaum_central_1983}.
This leads to an %
inverse-propensity weighted %
divergence loss:
\begin{align}
  \mathcal{L}_{\theta}^{\text{IPW}}  = \mathbb{E}_{(x, z) \sim p_{\text{data}}(x, z)} \left[ w(x, z) \cdot D\left(q_{\theta}(y \mid x, z),\, p(y \mid x, z)\right) \right], \label{eq:dd}
\end{align}
where the divergence \( D(q_\theta, p) \) is defined as $D(q_\theta, p) = \mathbb{E}_{y \sim q}[d(q_\theta, p)]$, for some pointwise divergence measure \( d(q_\theta, p) \) (see Section~\ref{subsec:divergence_choice} for the specific choice of divergence we adopt). 

Recent works~\citep{ma2024diffpo, mahajan2024empiricalanalysismodelselection} %
also use the %
inverse-propensity weights \( 1 / p_{\text{data}}(z \mid x) \). They train a propensity network \( g_{\omega} \) to obtain \( \hat{\pi}(x) = g_{\omega}(x) \) and form the weights \( 1 / \hat{\pi}(x) \) (or \( 1 / [1 - \hat{\pi}(x)] \)). When \( p_{\text{data}}(z \mid x) \) approaches zero, this leads to an excessively large weight which causes numerical instability. Common approaches to improve stability include: (i) \textit{truncating} the estimated propensity scores to a fixed interval, and (ii) \textit{trimming} the sample by discarding units with propensity scores outside that interval. Although both stabilize the IPW estimators, they introduce additional arbitrariness~\citep{ding2023coursecausalinference}. Moreover, improper clipping risks nullifying the intended reweighting effect. We identified such an issue in DiffPO \citep{ma2024diffpo}. Although it performs well on some datasets, its implementation incorrectly truncates \(1 / p_{\text{data}}(z \mid x)\) to values below one---despite the fact that \(1 / p_{\text{data}}(z \mid x) \geq 1\) by definition.\footnote{See DiffPO’s official implementation at commit \texttt{43ebb60}: 
\url{https://github.com/yccm/DiffPO/blob/43ebb6048dc09b0315e8f25db9b5d00a95b9b3e0/src/main_model.py\#L144-L147}
} 
This improper implementation nullifies the intended effect of propensity score reweighting, effectively reducing the objective in Equation~\ref{eq:loss_diffpo} to a standard diffusion loss without any reweighting.

\textbf{Implicit importance weighting via marginal sampling.} 
Given the pitfalls of existing IPW-based importance weighting, we now introduce a key result showing that the importance reweighting objective in Equation~\ref{eq:dd} can be reparameterized without explicitly computing weights through training a neural network for propensity score.

\begin{lemma}
\label{lemma:iwdd}
The importance-weighted loss in Equation~\ref{eq:dd} is equivalent to the expected divergence under the product of marginals:
\begin{align}
    \mathcal{L}_{\theta}^{\emph{\text{IWDD}}} = \mathbb{E}_{(x, z) \sim p_{\emph{\text{data}}}(x)p_{\emph{\text{rct}}}(z)} \left[ D\left(q_{\theta}(y \mid x, z),\, p(y \mid x, z)\right) \right]. \label{eq:iwdd}
\end{align}
\end{lemma}
\begin{proof}
Substituting the importance weight \( w(x, z) = \tfrac{p_{\text{data}}(x)p_{\text{rct}}(z)}{p_{\text{data}}(x, z)} \) into Equation~\ref{eq:dd}, we have:
\begin{align*}
\mathcal{L}_\theta^{\text{IPW}} 
&= \mathbb{E}_{(x, z) \sim p_{\text{data}}(x, z)} \left[  w(x, z) \cdot D(q_\theta(y \mid x, z),\, p(y \mid x, z)) \right] \\
&= \mathbb{E}_{(x, z) \sim p_{\text{data}}(x)p_{\text{rct}}(z)} \left[ D(q_\theta(y \mid x, z),\, p(y \mid x, z)) \right] = \mathcal{L}_\theta^{\text{IWDD}}.\qedhere
\end{align*}
\end{proof}
This equivalence enables us to apply the bias correction implicitly through a sampling adjustment: we sample \( x \sim p_{\text{data}}(x) \) and independently draw \( z \sim \mathrm{Bernoulli}(0.5)\). %
This %
yields the importance weight \( w(x, z) = 1 / p(z \mid x) \). This approach bypasses the need for propensity score estimation or weight clipping, as \( w(x, z) = 1 / p(z \mid x) \) is never explicitly computed. We will use the loss $\mathcal{L}_\theta^{\text{IWDD}}$ as the generator loss in Algorithm \ref{alg:iwdd}.

\textbf{Gradient variance advantage.} Although Lemma~\ref{lemma:iwdd} shows that \(\mathcal{L}_\theta^{\text{IWDD}}\) and \(\mathcal{L}_\theta^{\text{IPW}}\) have the same expectation, we find that gradient estimates under \(\mathcal{L}_\theta^{\text{IWDD}}\) can exhibit lower variance compared to those under \(\mathcal{L}_\theta^{\text{IPW}}\) when the weighting function is given by \(w(x,z)=p_{\mathrm{rct}}(z)/p_{\mathrm{data}}(z\mid x)\).

\begin{theorem}
Let \( \mathrm{Var}_{\mathrm{IWDD}} \) and \( \mathrm{Var}_{\mathrm{IPW}} \) denote the gradient covariance matrices under the IWDD loss (Equation~\ref{eq:iwdd}) and the IPW loss (Equation~\ref{eq:dd}), respectively. Then, the gradient variance under marginal sampling (IWDD) is upper bounded by that of the importance-weighted approach:
\[
\mathrm{Var}_{\emph{\text{IWDD}}} \preceq \mathrm{Var}_{\emph{\text{IPW}}}.
\]
\end{theorem}
\vspace{-0.5em}
\begin{proof}[Proof sketch]
Let \( g(x, z) = \nabla_\theta D(q_{\theta}(y \mid x, z), p(y \mid x, z)) \), \( w(x, z) = \tfrac{p_{\mathrm{rct}}(z)}{p_{\mathrm{data}}(z \mid x)} \). We compare \( \hat{g}_{\mathrm{IPW}}(x,z) = w(x,z)\,g(x,z) \), with \( (x,z) \sim p_{\text{data}}(x,z) \), versus \( \hat{g}_{\mathrm{IWDD}}(x,z) = g(x,z) \), where \( x \sim p_{\text{data}}(x) \), \( z \sim p_{\text{rct}}(z) \).

$
\operatorname{Var}_{\text{IPW}}-\operatorname{Var}_{\text{IWDD}}
   =\mathbb{E}_{p_{\mathrm{data}}(x,z)}
      \bigl[(w^{2}-w)\,g\,g^\top\bigr]
   \;\succeq\;0,$
because \(w^{2}-w=\tfrac1{4\pi(1-\pi)}-1\ge0\) with
\(\pi=p_{\mathrm{data}}(z=1\mid x)\). A detailed variance analysis proof is in Appendix~\ref{appendix:gradient-variance}.\qedhere
\end{proof}
\vspace{-0.4em}
Importantly, the IWDD formulation in Equation~\ref{eq:iwdd} achieves \textbf{lower gradient variance} than the IPW-based objective in Equation~\ref{eq:dd}. By removing explicit dependence on inverse propensity scores, it avoids the instability associated with high-variance weights and leads to more efficient optimization.

\vspace{-0.2em}
\subsubsection{Radomization-based Adjustment}

When sampling from $p_{\text{data}}(x)$ and $p_{\text{rct}}(z)$, we propose a novel randomization-based adjustment to the data used for training the generator \( q_{\theta} \), which improves performance compared to standard distillation procedures. The motivation for this adjustment stems from the observation that the gold standard for estimating treatment effects is the RCTs. Our approach aims to approximate an RCT-like setting by breaking the dependence between covariates \( X \) and treatment assignment \( Z \). Specifically, randomly \textbf{shuffling \(X\)} can eliminate existing associations between \(X\) and \(Z\) while preserving their marginal distributions. However, to ensure a balanced treatment assignment across individuals, we instead \textbf{sample \(Z\)} independently from a \(\text{Bernoulli}(0.5)\) distribution.
This setup aligns with our goal of predicting both \(Y_0\) and \(Y_1\) for each individual, effectively mirroring the conditions of an RCT.

While past literature has emphasized the importance of randomized designs and balancing methods to reduce bias in observational studies \citep{imbens_rubin_2015, rosenbaum_central_1983, stuart2010matching}, our adjustment represents a novel randomization procedure specifically tailored for improving the use of diffusion models for generative causal estimation.

\vspace{-0.2em}
\subsubsection{Choice of Divergence}
\label{subsec:divergence_choice}

We consider two representative divergence measures: the KL divergence and the Fisher divergence. Under both, the gradient of the divergence loss \(\mathcal{L}_{\theta}\) can be estimated via an alternating optimization procedure between a fake score network and a generator. The fake score network is trained to approximate the score of the generated response variable \(y\) given \(x, z \sim p_{\text{data}}(x, z)\), 
while the generator is trained to optimize \(q_{\theta}(y \mid x, z)\), which will ultimately be used for causal estimation.

We first implemented KL divergence-based distillation following the formulation of prior work \citep{wang2023prolificdreamer, luo2023diffinstruct, yin2024onestep}. While this approach improves the sampling efficiency of the pretrained model, the one-step generator distilled using KL divergence fails to improve upon the original model and often performs worse, as shown in the second row of Figure~\ref{fig:toy}. Moreover, it exhibits training instability and is prone to collapse in our synthetic data experiments. In contrast, a Fisher divergence objective combined with SiD-based gradient estimation \citep{zhou2024score,chen2025denoising} results in more stable training and consistently stronger empirical performance. Thus, we adopt Fisher divergence as the distillation objective in IWDD and use SiD to optimize it, leading to a one-step generator that achieves both high sampling efficiency and strong causal estimation accuracy.

During distillation, we alternate between updating the generator \(G_{\theta}\) and the fake score network \(f_{\psi}\). The generator \(G_\theta\) is trained on randomized pairs \((\tilde{x}, \tilde{z})\) obtained via randomization-based adjustment to distill a pretrained diffusion model \(f_\phi\), originally trained on the joint distribution \(p_{\text{data}}(y, x, z)\). The fake score network is trained on unadjusted observational pairs \((x, z)\sim p_{\text{data}}(x, z)\) to approximate the score of the generator's output distribution. By applying randomization-based adjustment only to the generator inputs while keeping the fake score network conditioned on observational data, we generalize the approach of \citet{zhou2024score} and define the two loss functions as follows:

\vspace{-1em}
\begin{align}
\text{Generator loss:} \quad \mathcal{L}_{\theta} &= w(t)\, \left( f_{\psi}(y_t \mid \tilde{x}, \tilde{z}) - f_{\phi}(y_t \mid \tilde{x}, \tilde{z}) \right)^\top \left( f_{\psi}(y_t \mid \tilde{x}, \tilde{z}) - y_g \right) \nonumber \\
+ (1 - \alpha)&\, w(t)\, \left\| f_{\phi}(y_t \mid \tilde{x}, \tilde{z}) - f_{\psi}(y_t \mid \tilde{x}, \tilde{z}) \right\|_2^2, \quad \text{where } (\tilde{x}, \tilde{z}) \sim p_{\text{data}}(\tilde{x}) p_{\text{rct}}(\tilde{z}). \label{eq:gen_loss}
\end{align}
\vspace{-0.3em}
\begin{equation}
\text{Fake loss:}\quad \mathcal{L}_{\psi} = \gamma(t)\, \left\| f_{\psi}(y_t \mid x, z) - y_g \right\|_2^2, \quad \text{where } (x, z) \sim p_{\text{data}}(x, z).
\label{eq:score_loss}
\end{equation}
Here, \( y_g \) is the sample generated by the one-step generator \( G_\theta \), and \( y_t = y_g + \sigma_t \epsilon_t \); \( f_\phi \), the teacher diffusion model, is pretrained to estimate \( \mathbb{E}[y_0 \mid y_t, x, z] \) and kept frozen during distillation; \( f_\psi \), the fake score network, is trained to match \( \mathbb{E}[y_g \mid y_t, x, z] \).
The full algorithm is in Algorithm~\ref{alg:iwdd}. Detailed training schedules and the weighting functions \(w(t)\) and \(\gamma(t)\) are provided in Appendix~\ref{sec:iwdd}.

\begin{algorithm}[h]
\caption{IWDD Training}
\label{alg:iwdd}
\begin{algorithmic}[1]
\Require Pretrained diffusion $f_\phi$, training data $\mathcal{D} = \{({x}_i, z_i)\}_{i=1}^n$,  batch size $B$
\State \textbf{Initialize:} $\theta \leftarrow \phi, \; \psi \leftarrow \phi$
\Repeat
\State Sample mini-batch indices $\mathcal{I} \subset \{1, \dots, n\}$ with $|\mathcal{I}| = B$
  \State $x \gets \{x_i\}_{i \in \mathcal{I}}$; \quad $z \gets \{z_i\}_{i \in \mathcal{I}}$

\State $\tilde{x}\gets \texttt{shuffle}(x)$;
\quad $\tilde{z} \sim \mathrm{Bernoulli}(0.5)\;\; $\Comment{\textcolor{blue}{Randomization-based adjustment}}
\State $y_g \gets G_\theta(\tilde{x}, \tilde z, \varepsilon),~\varepsilon \sim \mathcal{N}(0, \mathbf{I})$, and let $y_t = y_g + \sigma_t\epsilon_t,~\epsilon_t \sim \mathcal{N}(0, \mathbf{I})$
\State $\theta \gets \theta - \eta_\theta \nabla_\theta \mathcal{L}_\theta(y_g,\epsilon_t,\tilde{{x}}, \tilde{z})$ \Comment{Equation~\eqref{eq:gen_loss}}
    \State $\psi \gets \psi - \eta_\psi \nabla_\psi \mathcal{L}_\psi(y_g,\epsilon_t, {x}, z)$ \Comment{Equation~\eqref{eq:score_loss}}

\Until{Converge}
\Ensure Trained generator $G_\theta$
\end{algorithmic}
\end{algorithm}

\vspace{-0.1em}
\section{Experiments}
We evaluate IWDD on both synthetic and benchmark datasets. The synthetic study illustrates the effect of divergence choice and highlights IWDD's robustness to covariate shift. We then benchmark IWDD on standard causal inference datasets, comparing its performance against baselines in potential outcome prediction and heterogeneous treatment effect estimation.

\textbf{Performance metrics.} We evaluate model performance on 
estimation accuracy for both PO prediction and treatment effects estimation. For accuracy of PO predictions, we report the \textbf{Root Mean Squared Error (RMSE)}, defined as
$\text{RMSE} = \sqrt{\frac{1}{N}\sum_{i = 1}^N(\hat{y}_i - y_i)^2},$
where \( \hat{y}_i \) and \( y_i \) denote the predicted and true outcomes, respectively. Lower RMSE values indicate better predictive performance. For treatment effect estimation, we use the \textbf{Precision in Estimation of Heterogeneous Effect (PEHE)}:
$\epsilon_{\text{PEHE}} = \sqrt{\frac{1}{N}\sum_{i = 1}^N (\hat{\tau}(x_i) - \tau(x_i))^2},$
where \( \tau(x) = \mathbb{E}[Y(1) - Y(0) \mid X = x] \) is the Conditional Average Treatment Effect (CATE). Lower $\epsilon_{\text{PEHE}}$ values reflect better estimation accuracy. For win rates, we report percentages, measuring how often a method outperforms others.

\subsection{Synthetic Data Example: Advantage of IWDD in Out-of-Sample Estimation}
\label{sec:synthetic-iwdd}
\begin{wrapfigure}{r}{0.3\textwidth}
    \centering
    \vspace{-15pt} %
    \includegraphics[width=0.28\textwidth]{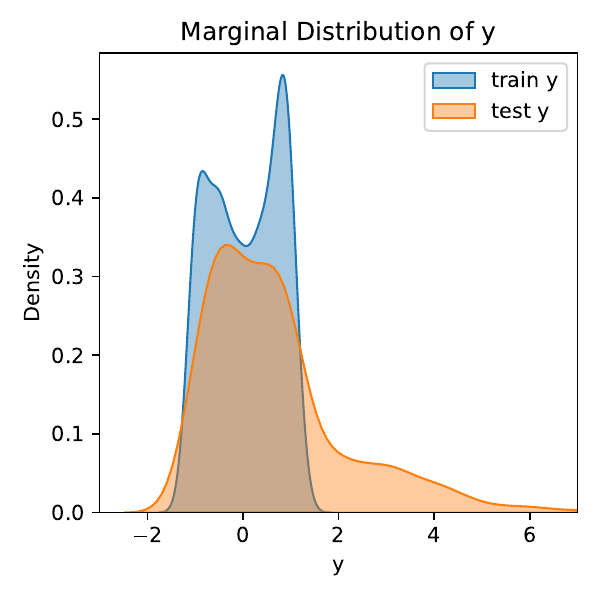}  %
    \vspace{-1\baselineskip}
    \caption{\small Marginal distributions of $y$ in training and testing sets. Due to the shift in treatment assignment, the induced distribution of $y$ differs across domains.}
    \vspace{-1em}
    \label{fig:density_y}
\end{wrapfigure}

We designed a synthetic data experiment to visualize the effectiveness of IWDD in addressing out-of-sample estimation challenges under distribution shift. The covariate distribution \( p(x) \sim \mathcal{N}(0, 1)\) and outcome model \(y = f(x, z) + \epsilon, \epsilon\sim \mathcal{N}(0,1)\) remain fixed across training and testing. The treatment assignment mechanism differs: in training, $z = \mathbf{1}\{x < -1\}$, resulting in only 16\% treated samples; in testing, \(z \sim \mathrm{Bernoulli}(0.5)\). Note that in testing, $z=1$ can occur when $x \ge -1$, creating $(x, z)$ pairs entirely unseen during training and contributing to the right tail of the marginal y distribution in Figure~\ref{fig:density_y}. This shift in $p(z\mid x)$ induces covariate shift, altering the marginal distribution of \(y\). We evaluate performance in two settings: in-sample (using the training distribution $p_{\text{data}}(x,z)$) and out-of-sample (under the test distribution $p_{\text{data}}(x)p_{\text{rct}}(z))$. Details of the data-generating process are in Appendix~\ref{appendix:toy-data}.

We denote the POs as $Y(0) = f(x, 0) + \epsilon$ %
and $Y(1) = f(x, 1) + \epsilon$, representing the untreated and treated outcomes. Figure~\ref{fig:toy} shows that the pretrained diffusion model (Row~1) performs well in-sample and on out-of-sample $Y(0)$.
However, it struggles with estimating out-of-sample $Y(1)$ due to the limited treated samples and unseen $(x, z)$ pairs in the training distribution, with an RMSE of 3.07. In contrast, IWDD (Row~3) improves estimation, with substantial gain observed in the out-of-sample $Y(1)$ predictions, reducing RMSE from 3.07 to 2.76. This demonstrates its effectiveness in addressing treatment imbalance and enhancing generalization under covariate shift.

We also implemented an alternative distillation strategy based on KL divergence \citep{yin2024onestep}. However, as shown in row 2 of Figure~\ref{fig:toy}, this approach not only fails to yield improvements but performs worse than the pretrained model. It is also prone to training instability and often collapses during optimization. Due to its superior empirical performance and robustness, we adopt Fisher divergence \citep{zhou2024score} as the preferred divergence in IWDD.

\begin{figure}[t] %
\centering
\renewcommand{\arraystretch}{1.5}

\begin{tabular}{@{}m{0.4cm} m{13cm}@{}}
& 
\makebox[13cm]{%
\hspace{-0.8cm} 
  $\overbrace{\hspace{5cm}}^{\text{In-sample}}%
  \hspace{1cm}%
   \overbrace{\hspace{5cm}}^{\text{Out-of-sample}}$
} \\[0.5ex]

\rotatebox{90}{Pretrain}  &
\includegraphics[width=\linewidth]{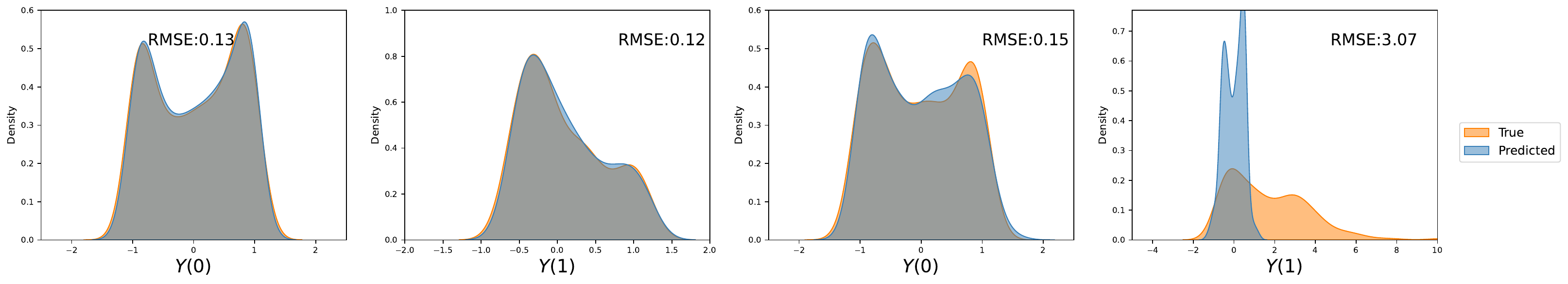} \\

\rotatebox{90}{\footnotesize \parbox{2.7cm}{\centering KL divergence\\ distillation}} &
\includegraphics[width=\linewidth]{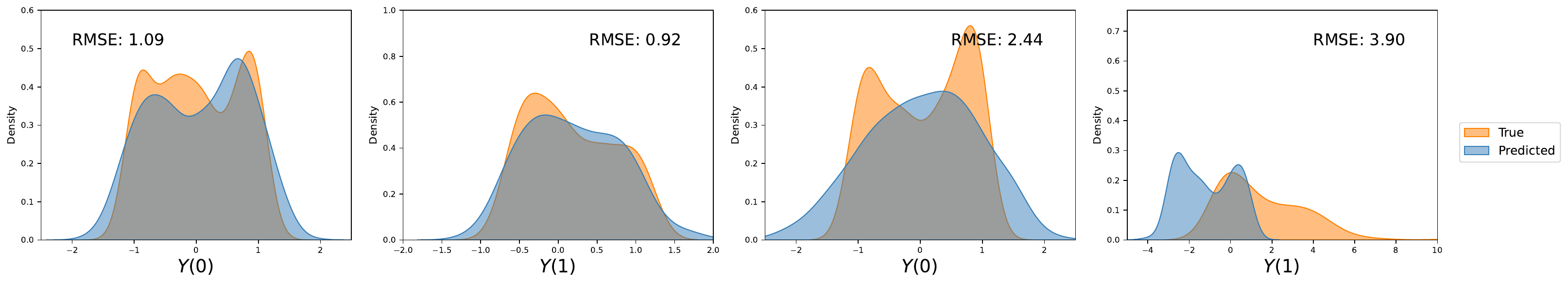} \\

\rotatebox{90}{\footnotesize \parbox{2.7cm}{\centering Fisher divergence\\ distillation (IWDD)}} &
\includegraphics[ width=\linewidth]{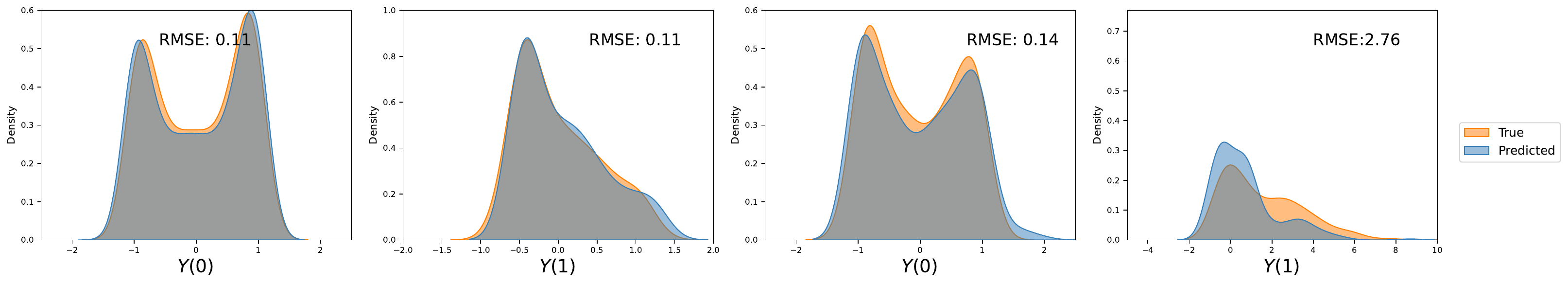}
\end{tabular}
\vspace{-5mm}
\caption{\small Synthetic data example: estimated potential outcome distributions $Y(0)$ and $Y(1)$ from different models. The pretrained diffusion model performs well in-sample and for $Y(0)$ out-of-sample, but struggles with $Y(1)$. IWDD improves estimation for out-of-sample $Y(1)$ while maintaining performance elsewhere.}%
\vspace{-3mm}
\label{fig:toy}
\end{figure}
\vspace{-0.3\baselineskip}

\subsection{Benchmarking on Standard Public Datasets}

\textbf{Datasets.} We evaluate our models on three widely used causal inference benchmarks. The \textbf{ACIC 2016} dataset includes 77 semi-synthetic datasets generated from real-world health care covariates, with 4802 observations and 55 covariates.\footnote{\url{https://jenniferhill7.wixsite.com/acic-2016/competition}} The \textbf{ACIC 2018} dataset consists of 12 semi-synthetic datasets with 10000 observations and 177 covariates, designed to test model robustness under complex treatment assignment and outcome mechanisms.\footnote{\url{https://www.synapse.org/\#!Synapse:syn11294478/wiki/}} The \textbf{IHDP}~\citep{gross2024ihdp} dataset contains 747 units and 25 covariates, and is based on a real randomized trial with simulated outcomes.

\textbf{Baselines.} While most of these methods are designed for CATE estimation, they can also be evaluated for PO prediction by assessing the RMSE of predicted outcomes \( Y(0) \) and \( Y(1) \). In our evaluations, we compare against S-learner and T-learner \citep{K_nzel_2019}, TARNet and CFR \citep{curth2021nonparametric}, GANITE \citep{yoon2018ganite}, and DiffPO \citep{ma2024diffpo} for potential outcome prediction, and further include DR-learner \citep{kennedy2023optimaldoublyrobustestimation}, RA-learner \citep{curth2021nonparametric}, and TEDVAE \citep{Zhang_Liu_Li_2021} for CATE estimation.

\textbf{Results and Discussion.}
\phantomsection \label{para:results} Results for ACIC 2018 are summarized in Tables~\ref{tab:rmse-summary-acic2018} and~\ref{tab:pehe-summary-acic2018}. Additional experiments were conducted on ACIC 2016 and IHDP. Results for ACIC 2016, including PO predictions (Table~\ref{tab:rmse-acic2016}) and treatment effect estimation (Table~\ref{tab:acic2016-pehe}), are provided in the Appendix~\ref{app:acic2016_exp_results}. Complete per-dataset results for ACIC 2018 are shown in Tables~\ref{tab:rmse-compact} and~\ref{tab:acic2018-pehe} (Appendix~\ref{app:acic2018_exp_results}). IHDP results are summarized in Tables~\ref{tab:ihdp_rmse} and~\ref{tab:ihdp_pehe} (Appendix~\ref{app:ihdp_exp_results}).

\begin{table}[t] %
\centering
\caption[\small Win rates and RMSE summary]{Win rates (\%) \footnotemark{} and mean RMSE for $Y(0)$, $Y(1)$ across 12 ACIC 2018 datasets. The best result across all methods is highlighted in \textbf{bold}, and the best result among the diffusion-based approaches (DiffPO, Pretrain, IWDD) is additionally marked with a $\star$.}
\label{tab:rmse-summary-acic2018}
\resizebox{.9\textwidth}{!}{
\begin{tabular}{lcccccccc}
\toprule
\textbf{Method} & \multicolumn{4}{c}{In-sample} & \multicolumn{4}{c}{Out-of-sample} \\
\cmidrule(lr){2-5} \cmidrule(lr){6-9}
& Win$_0$(\%) & RMSE$_0$ & Win$_1$(\%) & RMSE$_1$ & Win$_0$(\%) & RMSE$_0$ & Win$_1$(\%) & RMSE$_1$ \\
\midrule
T-learner     & 0 & 65.9 & 8.3 & 68.2 & 0 & 68.2 & 8.3 & 307.3 \\
S-learner     & 8.3 & 64.5 & 0 & 65.7 & 8.3 & 275.9 & 0 & 306.8 \\
TNet          & 0 & 0.857 & 0 & 0.973 & 8.3 & 1.018 & 0 & 1.120 \\
TARNet        & 0 & 0.855 & 0 & 0.967 & 0 & 1.022 & 0 & 1.113 \\
OffsetNet     & \textbf{75} & \textbf{0.786} & \textbf{33.3} & \textbf{0.889} & 0 & 1.114 & 0 & 1.218 \\
FlexTENet     & 8.3 & 0.876 & 16.7 & 0.987 & 0 & 1.058 & 0 & 1.149 \\
\cmidrule(lr){1-9}
DiffPO        & 8.3$^\star$ & 473.9 & 25$^\star$ & 975.2 & 16.7 & 472.7 & 25 & 975.2 \\
Pretrain      & 0 & 0.989 & 0 & 1.170 & 0 & 0.902 & 0 & 1.085 \\
\textbf{IWDD}          & 0 & 0.963$^\star$ & 16.7 & 1.038$^\star$ & \textbf{66.7} & \textbf{0.880}$^\star$ & \textbf{75} & \textbf{0.950}$^\star$ \\
\bottomrule
\end{tabular}}
\vspace{-3mm}
\end{table}
\footnotetext{Ties are counted for both methods when calculating win rates.}

The ACIC 2018 results reveal several key insights. First, IWDD consistently dominates the out-of-sample evaluations, both in PO prediction and CATE estimation, demonstrating strong generalization ability to unseen data. Although OffsetNet and FlexTENet \citep{curth2021inductivebiasesheterogeneoustreatment} perform well in-sample, they show weaker performance in the out-of-sample setting. This is due to how they are designed. OffsetNet uses a hard reparametrization approach, modeling the treatment effect as an additive offset to $\mathbb{E}[Y(0)\mid X]$, explicitly fitting the heterogeneity between POs. FlexTENet employs a multi-task learning architecture with shared and private subspaces to balance common and outcome-specific patterns. Both methods risk overfitting, limiting their ability to generalize. In contrast, IWDD generalizes well across diverse datasets by IPW-modulated distillation through randomization-based adjustment.

\begin{wraptable}{r}{0.55\textwidth}
\centering
\small
\vspace{-5mm}
\caption{\small Win rates (\%) and mean $\epsilon_{\text{PEHE}}$ on ACIC 2018}
\label{tab:pehe-summary-acic2018}
\resizebox{.53\textwidth}{!}{
\begin{tabular}{lcccc}
\toprule
 & \multicolumn{2}{c}{In-sample} & \multicolumn{2}{c}{Out-of-sample} \\
\cmidrule(lr){2-3} \cmidrule(lr){4-5}
& Win(\%) &  $\epsilon_{\text{PEHE}}$ & Win(\%) & $\epsilon_{\text{PEHE}}$ \\
\midrule
Causal Forest & 0\% & 13.452 & 8\% & 18.190 \\
T-learner     & 8\% & 15.175 & 8\% & 18.724 \\
S-learner     & 0\% & 8.613  & 0\% & 13.619 \\
TNet          & 0\% & 0.600  & 0\% & 0.574 \\
TARNet        & 0\% & 0.618  & 0\% & 0.619 \\
OffsetNet     & 0\% & 0.606  & 0\% & 0.606 \\
FlexTENet     & 0\% & 0.652  & 0\% & 0.646 \\
DRNet         & 0\% & 0.607  & 0\% & 0.606 \\
\cmidrule(lr){1-5}
DiffPO        & 33\% & 583.4  & 33\% & 589.1 \\
Pretrain      & 25\% & 0.318  & 33\% & 0.306 \\
\textbf{IWDD} & \textbf{50\%} & \textbf{0.308} & \textbf{58\%} & \textbf{0.301} \\
\bottomrule
\end{tabular}}
\vspace{-2mm}
\end{wraptable}

Second, IWDD shows numerical stability across all evaluated datasets, while other baselines exhibit varying levels of instability on different datasets. Notably, DiffPO is highly unstable: although it performs well on some ACIC 2018 datasets, it fails on all ACIC 2016 and IHDP datasets, and is thus excluded from those results due to lack of comparability\footnote{We use the official GitHub repository of DiffPO to produce the results: \url{https://github.com/yccm/DiffPO}.}. Although propensity weights, which should always exceed~1, are clipped to 0.9 for numerical stability, its DDPM-based schedules remains unstable in some settings. The EDM schedules we adopt, however, contributes to robustness. Importantly, our pretrained diffusion model already provides a strong and robust baseline, and IWDD further improves upon it, achieving the best overall performance.

\section{Discussion}
\label{sec:discussion}
This work introduces IWDD, a generative causal estimation framework that integrates diffusion models, importance weighting, and distillation to address covariate imbalance while enabling efficient one-step sampling. Empirical results on synthetic and real-world datasets demonstrate its robustness and strong generalization. Despite its strengths, IWDD has several limitations. It assumes no unmeasured confounding~\citep{he2024generalizing}—a strong assumption—and has so far been evaluated primarily on benchmark datasets. Future work will extend IWDD to more complex settings, including continuous treatments, discrete outcomes, and longitudinal data, while also exploring ways to relax identification assumptions and improve training efficiency.

{\small
\bibliographystyle{plainnat}  %
\bibliography{references.bib,zhougroup_ref.bib}     %
}

\newpage

\appendix

\section{Gradient Variance Comparison for IPW-based and IWDD Losses}
\label{appendix:gradient-variance}
We have
\begin{align}
  \mathcal{L}_{\theta}^{\text{IPW}}  = \mathbb{E}_{(x, z) \sim p_{\text{data}}(x, z)} \left[ w(x, z) \cdot D\left(q_{\theta}(y \mid x, z),\, p(y \mid x, z)\right) \right], 
\end{align}
and
\begin{align}
    \mathcal{L}_{\theta}^{{\text{IWDD}}} = \mathbb{E}_{(x, z) \sim p_{\text{data}}(x)p_{\text{rct}}(z)} \left[ D\left(q_{\theta}(y \mid x, z),\, p(y \mid x, z)\right) \right]. 
\end{align}

Let
\[
g(x,z) \;=\;\nabla_\theta\,D\bigl(\,q_\theta(y\mid x,z),p(y\mid x,z)\bigr),
\]
and the importance weight
\[
w(x,z)\;=\;\frac{p_{\text{rct}}(z)}{p_{\text{data}}(z\mid x)}
          \;=\;
          \begin{cases}
            \dfrac{\;1/2\;}{\pi(x)}        & z=1,\\[6pt]
            \dfrac{\;1/2\;}{1-\pi(x)}      & z=0.
          \end{cases}
\]

with $ \pi(x)\;:=\;p_{\text{data}}(z=1\mid x)$ denoting the propensity score, which by Assumption~\ref{assump: consist} satisfies $0<\pi(x)<1$.
We compare two stochastic‐gradient estimators:

1. Importance‐weighted using propensity score (sampling \((x,z)\sim p_{\rm data}(x,z)\), then weighting):
\[
\hat g_{\rm IPW}(x,z)=w(x,z)\,g(x,z).
\]

2. Marginal‐sampling (sampling \((x,z)\sim p_{\rm data}(x)\,p_{\rm rct}(z)\), no weight):
\[
\hat g_{\rm IWDD}(x,z)=g(x,z).
\]

In both cases the population gradient is

\[
G \;=\;\nabla_\theta\mathcal L_\theta
=\mathbb{E}_{p_{\rm data}(x,z)}[\,w\,g\,]
=\mathbb{E}_{p_{\rm data}(x)\,p_{\rm rct}(z)}[\,g\,].
\]

Variance under importance weighting:
\begin{align*}
\mathrm{Var}_{\rm IPW}
&=\mathbb{E}_{p_{\rm data}(x,z)}\bigl[\hat g_{\rm IPW}\,\hat g_{\rm IPW}^\top\bigr]
\;-\;G\,G^\top\\
&=\mathbb{E}_{p_{\rm data}(x,z)}\bigl[w(x,z)^2\,g(x,z)\,g(x,z)^\top\bigr]
\;-\;
\Bigl(\mathbb{E}_{p_{\rm data}(x,z)}\bigl[w(x,z)\,g(x,z)\bigr]\Bigr)
\Bigl(\mathbb{E}_{p_{\rm data}(x,z)}\bigl[w(x,z)\,g(x,z)\bigr]\Bigr)^\top.
\end{align*}

Variance under marginal sampling:

Since for any test function \(h\),
\[
\mathbb{E}_{p_{\rm data}(x)\,p_{\rm rct}(z)}[h(x,z)]
=\mathbb{E}_{p_{\rm data}(x,z)}\bigl[w(x,z)\,h(x,z)\bigr],
\]
we have
\begin{align*}
\mathrm{Var}_{\rm IWDD}
&=\mathbb{E}_{p_{\rm data}(x)\,p_{\rm rct}(z)}\bigl[g(x,z)\,g(x,z)^\top\bigr]
\;-\;G\,G^\top\\
&=\mathbb{E}_{p_{\rm data}(x,z)}\bigl[w(x,z)\,g(x,z)\,g(x,z)^\top\bigr]
\;-\;
\Bigl(\mathbb{E}_{p_{\rm data}(x,z)}\bigl[w(x,z)\,g(x,z)\bigr]\Bigr)
\Bigl(\mathbb{E}_{p_{\rm data}(x,z)}\bigl[w(x,z)\,g(x,z)\bigr]\Bigr)^\top.
\end{align*}

Comparing the two variances:
\begin{align}
\mathrm{Var}_{\rm IPW}
-\mathrm{Var}_{\rm IWDD}
=\mathbb{E}_{p_{\rm data}(x,z)}\bigl[(w^2-w)\;g\,g^\top\bigr],
\label{eq:var_gap_main}
\end{align}

Condition on \(x\) and write \(\Sigma(x)=\sum_{z}p_{\text{data}}(z\mid x)\,
g(x,z)g(x,z)^{\!\top}\succeq0\).
With \(\pi(x)=p_{\text{data}}(z=1\mid x)\) abbreviating the propensity
score, the inner expectation in Equation~\ref{eq:var_gap_main} becomes

\[
\begin{aligned}
\Delta(x)
  &:=\sum_{z=0}^{1}p_{\text{data}}(z\mid x)\,
       \bigl(w^2(x,z)-w(x,z)\bigr) \\[4pt]
  &=\pi(x)\!
     \Bigl[\bigl(\tfrac12/\pi(x)\bigr)^{\!2}-\tfrac12/\pi(x)\Bigr]
    +\bigl[1-\pi(x)\bigr]\!
     \Bigl[\bigl(\tfrac12/\!\bigl(1-\pi(x)\bigr)\bigr)^{\!2}
           -\tfrac12/\!\bigl(1-\pi(x)\bigr)\Bigr] \\[6pt]
  &=\frac{1}{4\,\pi(x)\bigl[1-\pi(x)\bigr]}-1
  \;\;\ge\;0,
\end{aligned}
\]
with equality \((\Delta(x)=0)\) \emph{iff} \(\pi(x)=\tfrac12\)
(observational data already perfectly balanced at that \(x\)).

Therefore
\[
\operatorname{Var}_{\text{IPW}}-\operatorname{Var}_{\text{IWDD}}
   \;=\;
   \mathbb{E}_{p_{\text{data}}(x)}
      \!\bigl[\,\Delta(x)\,\Sigma(x)\bigr]
   \;\succeq\;0,
\]
because each matrix \(\Sigma(x)\) is positive–semidefinite and
\(\Delta(x)\ge0\).

\section{Toy Data Generation Setup}
\label{appendix:toy-data}

We design a synthetic data generating mechanism to evaluate causal estimation methods under covariate shift, with a particular focus on the generalization of treatment effect estimation from observational (confounded) settings to randomized (unconfounded) settings. 
In our toy setup, covariate shift arises from changes in the treatment assignment mechanism: the covariate distribution \( p(x) \) is fixed across domains, but the conditional treatment distribution \( p(z \mid x) \) differs between training and test environments. 
This setup is conceptually related to the toy example designed for studying classical covariate shift scenarios in supervised learning \citep{sugiyama2007covariate}, where the input distribution \( p(x) \) changes for training and testing while the conditional outcome model \( p(y \mid x) \) remains invariant.

Let $x \in \mathbb{R}$ denote a scalar covariate sampled identically across both training and test environments from a standard normal distribution, $i.e.$, $x \sim \mathcal{N}(0, 1)$. The treatment assignment mechanism, however, differs between the training and test datasets.

In the training data, treatment is assigned deterministically based on the covariate:
\[
z = \mathbf{1}\{x < -1\},
\]
where $\mathbf{1}\{\cdot\}$ is the indicator function. This deterministic rule induces strong confounding, as treatment assignment is a function of $x$. This setup mimics observational studies in which patients with lower health scores or greater severity are more likely to receive treatment.

In contrast, the test data simulates a randomized controlled trial (RCT) setting, where treatment assignment is independent of covariates:
\[
z \sim \mathrm{Bernoulli}(0.5).
\]
This shift in treatment mechanism introduces a covariate distribution mismatch between $p_{\text{train}}(z \mid x)$ and $p_{\text{test}}(z \mid x)$, despite $p(x)$ remaining unchanged.

The outcome variable $y$ is generated from a nonlinear structural equation that depends on both the covariate $x$ and the treatment $z$:
\[
y = \sin(2x) + z \cdot \exp(x) + \epsilon, \quad \epsilon \sim \mathcal{N}(0, 0.01).
\]
The function $\sin(2x)$ introduces bounded nonlinear variation in the baseline outcome, while the multiplicative term $z \cdot \exp(x)$ captures heterogeneous treatment effects that increase exponentially with the covariate $x$. The additive noise term $\epsilon$ introduces mild stochasticity, simulating natural outcome variability.

This data generating process encapsulates several key challenges in real-world causal inference: (i) covariate-dependent confounding in observational data, (ii) covariate shift between observational and experimental domains, (iii) heterogeneous treatment effects, and (iv) nonlinear outcome surfaces. As such, it can be used for evaluating the robustness and generalization ability of causal inference methods under distributional mismatch.

\section{Ablation Study and Parameter Settings}
We conduct an ablation study to investigate the impact of the hyperparameter \( \alpha \) on the performance of IWDD. Tables~\ref{tab:diff_alpha} summarize results across a range of \( \alpha \in [0.3, 1.2] \) for the IHDP and ACIC 2018 datasets, respectively. 

On the IHDP benchmark, model performance is relatively stable for \( \alpha \in [0.3, 0.7] \), with lowest RMSE and Wasserstein distances achieved around \( \alpha = 0.7 \), while larger values such as \( \alpha = 1.2 \) lead to slightly degraded metrics. PEHE remains consistent across all \( \alpha \) values.

Similarly, on ACIC 2018, RMSE are minimized when \( \alpha \in [0.6, 1.2] \), particularly peaking at \( \alpha = 0.7 \) and \( \alpha = 1.2 \), whereas values below 0.5 or exactly at 1.0 show inferior results. Interestingly, PEHE shows very little sensitivity to \( \alpha \), remaining around 0.300 across the board.

Based on these observations, we use \( \alpha = 0.7 \) for overall robustness, though task-specific tuning between \( \alpha = 0.5 \) and \( \alpha = 1.2 \) may further optimize performance.

\begin{table}[t]
\centering

\vspace{1em}
\caption{Average out-of-sample evaluation results for the IWDD algorithm under different values of \( \alpha \) on IHDP (10 datasets) and ACIC 2018 (12 datasets)}
\label{tab:diff_alpha}

\vspace{1em}
\small %
\begin{tabular}{c ccc ccc}
\toprule
\multirow{2}{*}{\textbf{\( \alpha \)}} 
& \multicolumn{3}{c}{\textbf{IHDP}} 
& \multicolumn{3}{c}{\textbf{ACIC 2018}} \\
\cmidrule(lr){2-4} \cmidrule(lr){5-7}
& RMSE$_{y_0}$ & RMSE$_{y_1}$ & PEHE
& RMSE$_{y_0}$ & RMSE$_{y_1}$ & PEHE \\
\midrule
0.3 & 1.197 & 1.115 & 1.645 & 0.947 & 1.204 & 0.300 \\
0.4 & 1.126 & 1.042 & 1.644 & 0.923 & 1.152 & 0.301 \\
0.5 & 1.162 & 1.065 & 1.643 & 0.945 & 1.185 & 0.302 \\
0.6 & 1.399 & 1.305 & 1.643 & 0.895 & 1.019 & 0.300 \\
0.7 & 1.056 & 0.958 & 1.644 & 0.874 & 1.019 & 0.299 \\
1.0 & 1.193 & 1.090 & 1.643 & 1.082 & 1.209 & 0.300 \\
1.2 & 1.236 & 1.154 & 1.642 & 0.889 & 0.986 & 0.301 \\
\bottomrule
\end{tabular}
\normalsize %
\vspace{1em}
\end{table}

\section{Implementation Details}

\subsection{IWDD}
\label{sec:iwdd}

We implemented IWDD in PyTorch and conducted experiments on an NVIDIA RTX A5000 GPU. Below, we report the default settings of our model, though some hyperparameters may require minor tuning depending on the dataset.

\paragraph{Pretraining diffusion model}

The EDM loss is:
\[
\mathcal{L}_\phi = \mathbb{E}_{\sigma, y, {n}}\left[\lambda(\sigma)\|D_{\phi}(y_t ; \sigma, x, z)-{y}\|_2^2\right],
\]
where
$D_{\phi}(y;\sigma,x,z)=
    c_{\text{skip}}(\sigma)\,y_t \;+\;
    c_{\text{out}}(\sigma)\,
    f_{\phi}\!\bigl(c_{\text{in}}(\sigma)\,y_t,\,
                   c_{\text{noise}}(\sigma),x,z\bigr).$

We incorporate the diffusion scheduling approach from EDM \citep{karras2022edm}. We adopt the same network architecture and preconditioning scheme as EDM, including input scaling $c_{\text{in}}(\sigma) = 1 / \sqrt{\sigma^2 + \sigma^2_{\text{data}}}$, output scaling $c_{\text{out}}(\sigma) = \sigma \cdot \sigma_{\text{data}} / \sqrt{\sigma^2 + \sigma^2_{\text{data}}}$, skip scaling $c_{\text{skip}}(\sigma) = \sigma^2_{\text{data}} / (\sigma^2 + \sigma^2_{\text{data}})$, and noise conditioning $c_{\text{noise}}(\sigma) = \ln(\sigma)$. All relevant hyperparameters, including $\sigma_{\text{min}}$, $\sigma_{\text{max}}$, $\rho$, and $\mathcal{P}_{\text{mean}}$, are adopted from the EDM default configuration.

For data preprocessing, we followed the approach used in DiffPO (\url{https://github.com/yccm/DiffPO}). To guide the model, we used three causal masks as inputs: observational ($m_o$), target ($m_t$), and conditional ($m_c$) masks. $m_o$ indicates available observational data, $m_c$ marks conditioning variables $x$ and $a$, and $m_t$ marks observed outcomes $y$. The loss is computed only where $m_t = 1$.

\paragraph{Distillation} The distillation phase used hyperparameters consistent with the SiD implementation \citep{zhou2024score}. We have the loss functions for generator Eq.~\ref{eq:gen_loss} and fake score network Eq.~\ref{eq:score_loss}:
The weighting function $w(t)$ is defined as:
\[
w(t) = C/{\left\| y_g - f_\phi(y_t, t) \right\|_{1, \text{sg}}},
\]
where $y_t = y_g+\sigma_t\epsilon_t$, $C$ is the normalization constant in the structured setting, and $\| \cdot \|_{1,\text{sg}}$ denotes the stop-gradient L1 norm. The same function is used for $\gamma(t)$, following \citep{karras2022edm}.

The same hyperparameters are used as in the SiD implementation \citep{zhou2024score} are used. At each step, we sample $t \sim \text{Unif}[0, t_{\max}/1000]$ with $t_{\max} \in [0, 1000]$ and define the noise level using the $\rho$-parameterized EDM schedule:
\[
\sigma_t = \left(\sigma_{\text{max}}^{1/\rho} + (1 - t)\left(\sigma_{\text{min}}^{1/\rho} - \sigma_{\text{max}}^{1/\rho}\right)\right)^\rho,
\]
where $\sigma_{\text{min}} = 0.002$, $\sigma_{\text{max}} = 80$, and $\rho = 7.0$.

In the generation procedure $y_g = G_\theta(\sigma_{\text{init}},x, z, \epsilon), \quad \epsilon \sim \mathcal{N}(0, I)$,  $\sigma_{\text{init}}$ is set to $2.5$ and remains fixed throughout distillation and evaluation.

\subsection{Baselines}

We compared IWDD against several baselines implemented using publicly available codebases. DiffPO was included with its official implementation and default hyperparameters (\url{https://github.com/yccm/DiffPO}). CATENets-based estimators---including S-learner, T-learner, DR-learner, RA-learner, TNet, TARNet, OffsetNet, and FlexTENet---were adopted without modification from the CATENets repository (\url{https://github.com/AliciaCurth/CATENets/tree/main}). GANITE \citep{yoon2018ganite}, a generative-adversarial baseline for counterfactual prediction, was implemented using the MLforHealthLab repository (\url{https://github.com/vanderschaarlab/mlforhealthlabpub/tree/main/alg/ganite}). All baseline models were trained and evaluated using the same data splits, preprocessing pipelines, and evaluation metrics as IWDD.

\section{Experiments Results}
\label{app:experiments}

\subsection{ACIC 2016}
\label{app:acic2016_exp_results}

We present full potential outcome prediction results for ten selected datasets from the 77 ACIC 2016 datasets (Table~\ref{tab:rmse-acic2016}) and provide detailed treatment effect estimation results for three representative datasets (Table~\ref{tab:acic2016-pehe}).

\begin{table}[htbp]
\centering
\scriptsize
\caption{RMSE for POs $Y(0)$ and $Y(1)$ (in-sample and out-of-sample) across 10 ACIC 2016 datasets. The best result across all methods is highlighted in \textbf{bold}.}
\label{tab:rmse-acic2016}
\begin{tabular}{lcccccccccc}
\toprule
& \multicolumn{4}{c}{Dataset 1} & \multicolumn{4}{c}{Dataset 2} \\
\cmidrule(lr){2-5} \cmidrule(lr){6-9}
& RMSE$_{0,\text{in}}$ & RMSE$_{0,\text{out}}$ & RMSE$_{1,\text{in}}$ & RMSE$_{1,\text{out}}$
& RMSE$_{0,\text{in}}$ & RMSE$_{0,\text{out}}$ & RMSE$_{1,\text{in}}$ & RMSE$_{1,\text{out}}$ \\
\midrule
TNet         & 1.3069 & 1.4667 & 1.7358 & 1.8146 & 2.552 & 2.684 & 4.717 & 4.786 \\
TNet\_reg    & 1.3157 & 1.4827 & 1.5979 & 1.6294 & 2.442 & 2.448 & 3.295 & 3.210 \\
TARNet       & 1.2697 & 1.4360 & 1.4534 & 1.5463 & 2.343 & 2.320 & 2.941 & 2.872 \\
TARNet\_reg  & 1.2481 & 1.4148 & 1.3863 & 1.5026 & 2.251 & 2.213 & 2.607 & 2.578 \\
OffsetNet    & 1.2469 & 1.4302 & 1.4231 & 1.5514 & 2.229 & 2.198 & 2.409 & 2.306 \\
FlexTENet    & 1.1800 & 1.3420 & 1.2402 & 1.3453 & 2.140 & 1.995 & 2.522 & 2.388 \\
FlexTENet\_noortho & 1.3290 & 1.5117 & 1.5809 & 1.6845 & 2.377 & 2.415 & 3.004 & 2.932 \\
Pretrain     & 0.954  & 0.981  & 1.007  & 1.300  & 0.855 & 1.028 & 1.302 & 1.414 \\
IWDD         & \textbf{0.952} & \textbf{0.938} & \textbf{0.997} & \textbf{1.153} & \textbf{0.846} & \textbf{1.023} & \textbf{1.171} & \textbf{1.274} \\
\midrule
& \multicolumn{4}{c}{Dataset 3} & \multicolumn{4}{c}{Dataset 4} \\
\cmidrule(lr){2-5} \cmidrule(lr){6-9}
& RMSE$_{0,\text{in}}$ & RMSE$_{0,\text{out}}$ & RMSE$_{1,\text{in}}$ & RMSE$_{1,\text{out}}$
& RMSE$_{0,\text{in}}$ & RMSE$_{0,\text{out}}$ & RMSE$_{1,\text{in}}$ & RMSE$_{1,\text{out}}$ \\
\midrule
TNet         & 0.8678 & 1.0245 & 1.5526 & 1.6606 & 2.3082 & 2.2645 & 3.1790 & 3.8997 \\
TNet\_reg    & 0.8422 & 1.0140 & 1.2557 & 1.3984 & 2.2534 & 2.2559 & 2.3215 & 3.1703 \\
TARNet       & 0.8283 & 0.9890 & 1.2168 & 1.3560 & 2.1608 & 2.1516 & 2.1011 & 2.9622 \\
TARNet\_reg  & 0.7996 & 0.9720 & 1.0432 & 1.2079 & 2.0858 & 2.0906 & 2.0797 & 2.9229 \\
OffsetNet    & 0.7968 & 0.9852 & \textbf{0.9391} & \textbf{1.1258} & 2.1010 & 2.1452 & 2.1386 & 3.0270 \\
FlexTENet    & \textbf{0.7776} & 0.9690 & 1.1019 & 1.2620 & 2.0541 & 2.0146 & 2.2698 & 3.2725 \\
FlexTENet\_noortho & 0.8699 & 1.0731 & 1.2849 & 1.4463 & 2.1818 & 2.2343 & 2.3238 & 3.2638 \\
Pretrain     & 0.963  & \textbf{0.868}  & 1.753  & 1.428  & 0.918  & 1.001  & 1.090  & 0.979 \\
IWDD         & 0.954 & 0.871 & 1.735 & 1.250 & \textbf{0.909} & \textbf{0.986} & \textbf{1.132} & \textbf{0.890} \\
\midrule
& \multicolumn{4}{c}{Dataset 5} & \multicolumn{4}{c}{Dataset 6} \\
\cmidrule(lr){2-5} \cmidrule(lr){6-9}
& RMSE$_{0,\text{in}}$ & RMSE$_{0,\text{out}}$ & RMSE$_{1,\text{in}}$ & RMSE$_{1,\text{out}}$
& RMSE$_{0,\text{in}}$ & RMSE$_{0,\text{out}}$ & RMSE$_{1,\text{in}}$ & RMSE$_{1,\text{out}}$ \\
\midrule
TNet         & 1.834  & 2.081  & 2.747  & 2.824  & 1.212  & 1.481  & 2.564  & 2.676 \\
TNet\_reg    & 1.941  & 2.211  & 2.296  & 2.420  & 1.245  & 1.434  & 1.931  & 2.039 \\
TARNet       & 1.953  & 2.229  & 2.122  & 2.278  & 1.218  & 1.416  & 1.850  & 1.905 \\
TARNet\_reg  & 1.921  & 2.208  & 1.977  & 2.187  & 1.159  & 1.392  & 1.651  & 1.770 \\
OffsetNet    & 1.915  & 2.260  & 2.035  & 2.238  & 1.139  & 1.392  & 1.545  & 1.719 \\
FlexTENet    & 1.747  & 2.021  & 1.797  & 1.961  & 1.094  & 1.325  & 1.623  & 1.725 \\
FlexTENet\_noortho & 2.014  & 2.329  & 2.134  & 2.363  & 1.252  & 1.465  & 1.963  & 2.030 \\
Pretrain     & 0.907  & \textbf{1.235}  & \textbf{0.991}  & 1.058  & 0.964 & \textbf{0.993} & 1.194 &  \textbf{1.290} \\
IWDD         & \textbf{0.875} & 1.305 & 1.007 & \textbf{1.031} & \textbf{0.962} & 1.027 & \textbf{1.189} &1.384 \\
\midrule
& \multicolumn{4}{c}{Dataset 7} & \multicolumn{4}{c}{Dataset 26} \\
\cmidrule(lr){2-5} \cmidrule(lr){6-9}
& RMSE$_{0,\text{in}}$ & RMSE$_{0,\text{out}}$ & RMSE$_{1,\text{in}}$ & RMSE$_{1,\text{out}}$
& RMSE$_{0,\text{in}}$ & RMSE$_{0,\text{out}}$ & RMSE$_{1,\text{in}}$ & RMSE$_{1,\text{out}}$ \\
\midrule
T-learner    & 2.328 & 1.971 & 3.233 & 2.827 & 2.395 & 2.854 & 2.540 & 3.695 \\
S-learner    & 2.301 & 1.956 & 3.429 & 3.006 & 2.408 & 2.883 & 2.523 & 3.716 \\
TNet         & 2.075 & 2.534 & 3.548 & 4.173 & 2.279 & 3.985 & 3.479 & 3.932 \\
TNet\_reg    & 1.496 & 3.231 & 2.853 & 4.074 & 2.081 & 3.807 & 3.052 & 3.644 \\
TARNet       & 1.486 & 3.204 & 2.754 & 3.983 & 1.899 & 3.858 & 2.576 & 3.308 \\
TARNet\_reg  & 1.459 & 3.172 & 2.633 & 3.874 & 1.875 & 3.840 & 2.430 & 3.248 \\
OffsetNet    & 1.596 & 3.311 & 2.671 & 3.853 & 1.897 & 3.706 & 2.422 & 3.111 \\
FlexTENet    & 1.424 & 3.133 & 2.533 & 3.740 & 1.778 & 3.340 & 2.402 & 3.062 \\
FlexTENet\_noortho & 1.538 & 3.243 & 2.796 & 4.039 & 2.004 & 3.841 & 2.812 & 3.624 \\
Pretrain     & 0.923 & 0.971 & 1.246 & 1.230 & 1.038 & 1.315 & 1.452 & 1.877 \\
IWDD         & \textbf{0.920} & \textbf{0.956} & \textbf{1.231} & \textbf{1.135} & \textbf{0.968} & \textbf{1.100} & \textbf{1.352} & \textbf{1.591} \\
\midrule
& \multicolumn{4}{c}{Dataset 9} & \multicolumn{4}{c}{Dataset 10} \\
\cmidrule(lr){2-5} \cmidrule(lr){6-9}
& RMSE$_{0,\text{in}}$ & RMSE$_{0,\text{out}}$ & RMSE$_{1,\text{in}}$ & RMSE$_{1,\text{out}}$
& RMSE$_{0,\text{in}}$ & RMSE$_{0,\text{out}}$ & RMSE$_{1,\text{in}}$ & RMSE$_{1,\text{out}}$ \\
\midrule
TNet         & 2.835 & 3.048 & 3.852 & 3.861 & 1.724 & 2.037 & 3.708 & 3.858 \\
TNet\_reg    & 2.843 & 3.089 & 3.576 & 3.790 & 1.735 & 1.959 & 2.573 & 2.874 \\
TARNet       & 2.762 & 3.057 & 3.305 & 3.565 & 1.644 & 1.970 & 2.252 & 2.508 \\
TARNet\_reg  & 2.717 & 3.056 & 3.168 & 3.471 & 1.579 & 1.880 & 2.051 & 2.304 \\
OffsetNet    & 2.685 & 3.035 & 3.100 & 3.501 & 1.613 & 1.973 & 2.187 & 2.502 \\
FlexTENet    & 2.508 & 2.821 & 3.037 & 3.324 & 1.505 & 1.659 & 2.058 & 2.332 \\
FlexTENet\_noortho & 2.748 & 3.081 & 3.289 & 3.568 & 1.688 & 2.019 & 2.424 & 2.735 \\
Pretrain     & 1.095 & \textbf{1.066} & \textbf{0.904} & 1.006 & 0.950 & 1.068 & 1.437 & 1.407  \\
IWDD         & \textbf{1.085} & 1.198 & 0.906 & \textbf{0.962}  & \textbf{0.931} & \textbf{1.062} & \textbf{1.392} & \textbf{1.396}\\
\bottomrule
\end{tabular}
\end{table}

\begin{table}[htbp]
\centering
\caption{PEHE (in-sample and out-of-sample) on ACIC 2016-2, 2016-7, and 2016-26}
\label{tab:acic2016-pehe}
\begin{tabular}{lcccccc}
\hline
\textbf{Algorithm} & \multicolumn{2}{c}{\textbf{2016-2}} & \multicolumn{2}{c}{\textbf{2016-7}} & \multicolumn{2}{c}{\textbf{2016-26}} \\
\cmidrule(lr){2-3} \cmidrule(lr){4-5} \cmidrule(lr){6-7}
& PEHE$_{\text{in}}$ & PEHE$_{\text{out}}$ & PEHE$_{\text{in}}$ & PEHE$_{\text{out}}$ & PEHE$_{\text{in}}$ & PEHE$_{\text{out}}$ \\
\hline
Causal Forest (CF)   & \textbf{0.322} & \textbf{0.320} & 3.832 & 2.969 & 2.613 & 3.234 \\
T-learner            & 1.071 & 1.054 & 3.605 & 2.790 & 2.493 & 3.015 \\
S-learner            & 0.906 & 0.873 & 3.902 & 3.075 & 2.721 & 3.357 \\
TNet                 & 4.114 & 4.315 & 4.180 & 4.377 & 3.534 & 4.219 \\
TNet\_reg            & 2.528 & 2.560 & 2.961 & 5.256 & 3.218 & 3.984 \\
TARNet               & 2.157 & 2.164 & 2.899 & 5.230 & 2.695 & 3.485 \\
TARNet\_reg          & 1.605 & 1.577 & 2.740 & 5.104 & 2.506 & 3.268 \\
OffsetNet            & 1.454 & 1.464 & 2.865 & 5.221 & 2.600 & 3.248 \\
FlexTENet            & 1.466 & 1.450 & 2.660 & 5.026 & 2.544 & 3.569 \\
FlexTENet\_noortho   & 2.205 & 2.199 & 2.931 & 5.258 & 2.930 & 3.833 \\
DRNet                & 2.402 & 2.552 & 3.611 & 5.890 & 3.215 & 3.916 \\
DRNet\_TAR           & 1.594 & 1.599 & 3.154 & 5.427 & 2.825 & 3.526 \\
Pretrain             & 0.849 & 0.851 & \textbf{1.128} & \textbf{1.034} & \textbf{1.377} & \textbf{1.562} \\
\textbf{IWDD}        & 0.852 & 0.851 & 1.130 & 1.036 & 1.387 & \textbf{1.562} \\
\hline
\end{tabular}
\end{table}

\subsection{ACIC 2018}
\label{app:acic2018_exp_results}

We present the complete RMSE results for each of the 12 ACIC 2018 datasets in Table~\ref{tab:rmse-compact}, with in-sample and out-of-sample RMSE reported for the POs $Y(0)$ and $Y(1)$. The full PEHE results for each dataset are provided in Table~\ref{tab:acic2018-pehe}.

\vspace{-1em}
\renewcommand{\arraystretch}{0.93}
\begin{table}[htbp]
\centering
\scriptsize
\caption{RMSE for potential outcomes $Y(0)$ and $Y(1)$ (in-sample and out-of-sample) across 12 ACIC 2018 datasets. Each 4-column block corresponds to one dataset, labeled by its ID. The best result across all methods is highlighted in \textbf{bold}, and the best result among the diffusion-based approaches (DiffPO, Pretrain, IWDD) is additionally marked with a $\star$.}
\label{tab:rmse-compact}
\begin{tabular}{lccccccccc}
\toprule
& \multicolumn{4}{c}{\texttt{e36aca1030264e638452ea4053cbb42c}} & \multicolumn{4}{c}{\texttt{d4ae3280e4e24ca395533e429726fafc}} \\
\cmidrule(lr){2-5} \cmidrule(lr){6-9}
& RMSE$_{0, \text{in}}$ & RMSE$_{0, \text{out}}$ & RMSE$_{1, \text{in}}$ & RMSE$_{1, \text{out}}$ 
& RMSE$_{0, \text{in}}$ & RMSE$_{0, \text{out}}$ & RMSE$_{1, \text{in}}$ & RMSE$_{1, \text{out}}$ \\
\midrule
T-learner   & 273.048 & 353.999 & 273.043 & 353.783 & 2.138 & 2.165 & 2.067 & 2.167 \\
S-learner   & 264.114 & 353.674 & 264.136 & 353.746 & 2.241 & 2.280 & 2.202 & 2.293 \\
TNet        & 0.872   & 1.119   & 0.877   & 1.118  & 0.520  & 0.545  & 0.528  & 0.565 \\
TARNet      & 0.874   & 1.113   & 0.887   & 1.130  & 0.532  & 0.550  & 0.515  & 0.573 \\
OffsetNet   & 0.836   & 1.163   & 0.852   & 1.204  &\textbf{ 0.472}  & 0.614  & 0.554  & 0.681 \\
FlexTENet   & \textbf{0.818}  & 1.155   &\textbf{ 0.836}   & 1.185  & 0.475  & 0.565  & \textbf{0.486}  & 0.607 \\
\arrayrulecolor{gray!60}\cmidrule(lr){1-9}\arrayrulecolor{black}
DiffPO  & 1.025 & 1.043 & 1.025 & 1.045  & 0.509 & 0.454 & 0.509$^\star$  & \textbf{0.454}$^\star$  \\
Pretrain    & 1.019 & 1.037 & 1.021 & 1.040 & 0.517 & 0.459 & 1.697 & 1.697       \\
IWDD        & 1.012$^\star$ & \textbf{1.028}$^\star$ & 1.014$^\star$ & \textbf{1.030}$^\star$ & 0.503$^\star$  & \textbf{0.447}$^\star$  & 1.786 & 1.782     \\
\midrule
& \multicolumn{4}{c}{\texttt{d1546da12d8e4daf8fe6771e2187954d}} & \multicolumn{4}{c}{\texttt{ae51149d38ce42609e00bf5701e4fe88}} \\
\cmidrule(lr){2-5} \cmidrule(lr){6-9}
& RMSE$_{0, \text{in}}$ & RMSE$_{0, \text{out}}$ & RMSE$_{1, \text{in}}$ & RMSE$_{1, \text{out}}$ 
& RMSE$_{0, \text{in}}$ & RMSE$_{0, \text{out}}$ & RMSE$_{1, \text{in}}$ & RMSE$_{1, \text{out}}$ \\
\midrule
T-learner   & 48.249 & 2763.573 & 50.768 & 2829.238 & 17.130 & 18.984 & 17.288 & 18.481 \\
S-learner   & 47.393 & 2763.576 & 48.612 & 2829.269 & 16.731 & 18.683 & 16.802 & 18.335 \\
TNet        & 0.092  & 1.282  & 0.179  & 1.322    & 1.017  & 1.124  & 1.003  & 1.073 \\
TARNet     & 0.058  & 1.278  & 0.109  & 1.311    & 1.038  & 1.153  & 0.993  & 1.069 \\
OffsetNet   & \textbf{0.057}  & 1.279  & \textbf{0.058}  & 1.309   & 0.976  & 1.306  & 0.890  & 1.200 \\
FlexTENet   & 0.884  & 1.055    & 1.076  & 1.203    & 1.029  & 1.296  & 0.927  & 1.117 \\
\arrayrulecolor{gray!60}\cmidrule(lr){1-9}\arrayrulecolor{black}
DiffPO  & 1.105 & 0.027 & 1.131 & 0.026  & \textbf{0.958}$^\star$ & \textbf{1.050}$^\star$ & \textbf{0.957}$^\star$ & \textbf{1.051}$^\star$ \\
Pretrain    & 1.105  & 0.028    & 1.131  & 0.029    & 0.996  & 1.067  & 0.980  & 1.052 \\
IWDD       & 1.105 & \textbf{0.026}$^\star$ & 1.131 & \textbf{0.026} $^\star$         & 0.997 & 1.067 & 0.979 & \textbf{1.051 }$^\star$  \\
\midrule
& \multicolumn{4}{c}{\texttt{ac6e494cbc254dc599be26a2a17f229c}} & \multicolumn{4}{c}{\texttt{9333a461d3944d089ef60cdf3b88fd40}} \\
\cmidrule(lr){2-5} \cmidrule(lr){6-9}
& RMSE$_{0, \text{in}}$ & RMSE$_{0, \text{out}}$ & RMSE$_{1, \text{in}}$ & RMSE$_{y_1, \text{out}}$ 
& RMSE$_{0, \text{in}}$ & RMSE$_{0, \text{out}}$ & RMSE$_{1, \text{in}}$ & RMSE$_{1, \text{out}}$ \\
\midrule
T-learner   & 15.834 & 16.559 & 15.882 & 16.962 & 7.754 & 8.281 & 9.244 & 9.427 \\
S-learner   & 15.503 & 16.432 & 15.391 & 16.621 & 7.585 & 8.161 & 8.689 & 8.943 \\
TNet        & 0.993  & 1.050  & 1.066  & 1.147  & 1.008 & 1.073 & 1.184 & 1.214 \\
TARNet      & 0.996  & 1.059  & 1.073  & 1.154  & 1.000 & 1.087 & 1.195 & 1.244 \\
OffsetNet   & \textbf{0.903}  & 1.144  & 1.022  & 1.313  &\textbf{0.910} & 1.191 & 1.073 & 1.333 \\
FlexTENet   & 0.927  & 1.096  & 1.033  & 1.271  & 0.949 & 1.149 & 1.119 & 1.265 \\
\arrayrulecolor{gray!60}\cmidrule(lr){1-9}\arrayrulecolor{black}
DiffPO  & 1.004$^\star$ & 1.020 & \textbf{1.004}$^\star$ & 1.021 & 0.995$^\star$ & \textbf{1.056}$^\star$& \textbf{0.999}$^\star$ & \textbf{1.058$^\star$} \\
Pretrain    & 1.147  & 1.126  & 1.239  & 1.213  & 1.070 & 1.113 & 1.215 & 1.252 \\
IWDD        & 1.010  & \textbf{0.970}$^\star$ & 1.045 & \textbf{0.994}$^\star$ &    1.016   &   1.101    &   1.125    &   1.232    \\
\midrule
& \multicolumn{4}{c}{\texttt{8ff38d337ec842dab1b8c01076e24816}} & \multicolumn{4}{c}{\texttt{74420a1794304013bb7a5a8f61994d71}} \\
\cmidrule(lr){2-5} \cmidrule(lr){6-9}
& RMSE$_{0, \text{in}}$ & RMSE$_{0, \text{out}}$ & RMSE$_{1, \text{in}}$ & RMSE$_{1, \text{out}}$ 
& RMSE$_{0, \text{in}}$ & RMSE$_{0, \text{out}}$ & RMSE$_{1, \text{in}}$ & RMSE$_{1, \text{out}}$ \\
\midrule
T-learner   & 21.940 & 21.629 & 21.018 & 21.543 & 182.881 & 187.618 & 186.001 & 187.741 \\
S-learner   & 21.199 & 21.207 & 20.632 & 21.232 & 179.791 & 187.670 & 181.017 & 187.627 \\
TNet        & 1.079  & 1.098  & 1.042  & 1.090  & 1.008   & 1.047   & 1.026   & 1.050 \\
TARNet      & 1.070  & 1.077  & 1.042  & 1.096  & 1.014   & 1.037   & 1.039   & 1.045 \\
OffsetNet   & \textbf{0.970}  & 1.218  & \textbf{0.951}  & 1.244  & \textbf{0.909}   & 1.168   & \textbf{0.916}   & 1.195 \\
FlexTENet   & 1.014  & 1.106  & 0.984  & 1.127  & 0.923   & 1.141   & 0.944   & 1.124 \\
\arrayrulecolor{gray!60}\cmidrule(lr){1-9}\arrayrulecolor{black}
DiffPO  & 2344.775 & 2302.331 & 4969.306 & 4861.299 & 1.050 & 1.007 & 1.051 & 1.004\\
Pretrain    & 1.063  & 1.063  & 1.068  & 1.068  & 1.045   & 0.990   & 1.045   & 0.990 \\
IWDD     & 1.046$^\star$& \textbf{1.050}$^\star$ & 1.049$^\star$ & \textbf{1.052}$^\star$& 1.023 $^\star$& \textbf{0.970}$^\star$ & 1.024$^\star$ & \textbf{0.970}$^\star$ \\
\midrule
& \multicolumn{4}{c}{\texttt{110f6dc8583c456ea0dd242d5d598497}} & \multicolumn{4}{c}{\texttt{3ebc51612e034ff99e8632a228dae430}} \\
\cmidrule(lr){2-5} \cmidrule(lr){6-9}
& RMSE$_{0, \text{in}}$ & RMSE$_{0, \text{out}}$ & RMSE$_{1, \text{in}}$ & RMSE$_{y_1, \text{out}}$ 
& RMSE$_{0, \text{in}}$ & RMSE$_{0, \text{out}}$ & RMSE$_{1, \text{in}}$ & RMSE$_{1, \text{out}}$ \\
\midrule
T-learner   & 0.404 & 0.416 & \textbf{0.716} & \textbf{0.731} & 33.291 & 34.316 & 36.608 & 36.009 \\
S-learner   & \textbf{0.395} & \textbf{0.415} & 0.727 & 0.747 & 33.082 & 34.230 & 34.997 & 35.051 \\
TNet        & 0.723 & 0.754 & 1.271 & 1.313 & 1.006 & 1.029 & 1.126 & 1.147 \\
TARNet      & 0.727 & 0.756 & 1.264 & 1.302 & 1.017 & 1.034 & 1.119 & 1.102 \\
OffsetNet   & 0.669 & 0.883 & 1.186 & 1.384 & \textbf{0.920} & 1.145 & 1.016 & 1.210 \\
FlexTENet   & 0.685 & 0.810 & 1.213 & 1.343 & 0.963 & 1.096 & 1.059 & 1.151 \\
\arrayrulecolor{gray!60}\cmidrule(lr){1-9}\arrayrulecolor{black}
DiffPO & 1.272 & 1.277 & 1.874 & 1.885& 128.388 & 94.869 & 133.157 & 99.067 \\
Pretrain    & 0.760 & 0.745 & 1.505 & 1.481 & 0.998  & 1.020  & 0.997  & 1.019 \\
IWDD        & 0.753$^\star$ & 0.717$^\star$ & 1.486$^\star$ & 1.402$^\star$  & 0.995$^\star$  & \textbf{1.018}$^\star$  & \textbf{0.996}$^\star$  & \textbf{1.018} $^\star$       \\
\midrule
& \multicolumn{4}{c}{\texttt{5a147c7e542a4ea5b22da127b654666b}} & \multicolumn{4}{c}{\texttt{5ad181455e954bcba44743e1f2d7824e}} \\
\cmidrule(lr){2-5} \cmidrule(lr){6-9}
& RMSE$_{0, \text{in}}$ & RMSE$_{0, \text{out}}$ & RMSE$_{1, \text{in}}$ & RMSE$_{y_1, \text{out}}$ 
& RMSE$_{0, \text{in}}$ & RMSE$_{0, \text{out}}$ & RMSE$_{1, \text{in}}$ & RMSE$_{1, \text{out}}$ \\
\midrule
T-learner   & 129.082 & 142.823 & 136.431 & 144.381 & 59.179 & 62.521 & 69.447 & 68.557 \\
S-learner   & 127.209 & 142.641 & 129.352 & 142.609 & 58.596 & 62.059 & 65.345 & 65.455 \\
TNet        & 0.957   & 1.056   & 1.133   & 1.195   & 1.005  & \textbf{1.066}  & 1.132  & 1.125 \\
TARNet      & 0.943   & 1.057   & 1.124   & 1.187   & 0.995  & 1.068  & 1.141  & 1.137 \\
OffsetNet   & \textbf{0.870}   & 1.128   & 1.104   & 1.356   & \textbf{0.941}  & 1.136  & \textbf{1.046}  & 1.181 \\
FlexTENet   & 0.892   & 1.088   & 1.082   & 1.232   & 0.958  & 1.138  & 1.084  & 1.166 \\
\arrayrulecolor{gray!60}\cmidrule(lr){1-9}\arrayrulecolor{black}
DiffPO & 3204.518 & 3267.525 & 6589.601 & 6733.434 & 1.119 & 1.118 & 1.380 & 1.141 \\
Pretrain    & 0.993 & 1.058 & 0.990 & 1.056        & 1.157 & 1.117 & 1.148 & 1.123      \\
IWDD       & 0.987$^\star$  & \textbf{1.053}$^\star$  & \textbf{0.989}$^\star$  & \textbf{1.053}$^\star$         & 1.109$^\star$ & 1.113$^\star$ & 1.110$^\star$  & \textbf{1.113}$^\star$       \\
\bottomrule

\addlinespace
\end{tabular}
\end{table}

\renewcommand{\arraystretch}{1}

\begin{table}[htbp]
\centering
\footnotesize
\caption{PEHE (in-sample and out-of-sample) across 12 ACIC 2018 datasets. Each 2-column block corresponds to one dataset, labeled by its ID (truncated). The best performance across all methods is marked in \textbf{bold}, and the best among diffusion-based methods is marked with $\star$.}
\label{tab:acic2018-pehe}
\begin{tabular}{lcccccc}
\toprule
 & \multicolumn{2}{c}{\texttt{e36aca10...}} & \multicolumn{2}{c}{\texttt{d4ae3280...}} & \multicolumn{2}{c}{\texttt{d1546da1...}} \\
\cmidrule(lr){2-3} \cmidrule(lr){4-5} \cmidrule(lr){6-7}
& PEHE$_{\text{in}}$ & PEHE$_{\text{out}}$ & PEHE$_{\text{in}}$ & PEHE$_{\text{out}}$ & PEHE$_{\text{in}}$ & PEHE$_{\text{out}}$ \\
\midrule
Causal Forest & 15.577 & 12.722 & 0.690 & 0.684 & 4.354 & 65.782 \\
T-learner     & 25.286 & 15.863 & 0.696 & 0.679 & 6.140 & 65.813 \\
S-learner     & 5.704  & 4.951  & 1.609 & 1.611 & 1.966 & 65.705 \\
TNet          & 0.370  & 0.343  & 0.331  & 0.318  & 0.165 & 0.153 \\
TARNet        & 0.416  & 0.410  & 0.385  & 0.412  & 0.125 & 0.135 \\
OffsetNet     & 0.437  & 0.437  & 0.361  & 0.349  & 0.003 & 0.030 \\
FlexTENet     & 0.397  & 0.397  & 0.313  & 0.311  & 0.692 & 0.687 \\
DRNet         & 0.406  & 0.400  & 0.325  & 0.356  & 0.177 & 0.196 \\
\cmidrule(lr){1-7}
DiffPO  & 0.032 & 0.053  & \textbf{0.025}$^\star$ & \textbf{0.017}$^\star$  & 0.031 & 0.009\\
Pretrain      & \textbf{0.011}$^\star$ & \textbf{0.011}$^\star$ & 1.744 & 1.742       & 0.026 & \textbf{0.001}$^\star$ \\
IWDD          & \textbf{0.011}$^\star$ & \textbf{0.011}$^\star$ &    1.741     &     1.741   & \textbf{0.001}$^\star$ & \textbf{0.001}$^\star$     \\
\midrule
\addlinespace
 & \multicolumn{2}{c}{\texttt{ae51149...}} & \multicolumn{2}{c}{\texttt{ac6e494...}} & \multicolumn{2}{c}{\texttt{9333a461...}} \\
\cmidrule(lr){2-3} \cmidrule(lr){4-5} \cmidrule(lr){6-7}
& PEHE$_{\text{in}}$ & PEHE$_{\text{out}}$ & PEHE$_{\text{in}}$ & PEHE$_{\text{out}}$ & PEHE$_{\text{in}}$ & PEHE$_{\text{out}}$ \\
\midrule
Causal Forest & 9.740 & 9.686 & 7.499 & 7.505 & 7.480 & 7.481 \\
T-learner     & 9.722 & 9.602 & 7.597 & 7.545 & 7.531 & 7.520 \\
S-learner     & 7.823 & 7.744 & 5.188 & 5.167 & 6.115 & 6.099 \\
TNet          & 0.673 & 0.657 & 0.686 & 0.679 & 0.995 & 0.983 \\
TARNet        & 0.744 & 0.737 & 0.723 & 0.739 & 1.030 & 1.033 \\
OffsetNet     & 0.798 & 0.779 & 0.739 & 0.750 & 0.948 & 0.950 \\
FlexTENet     & 0.841 & 0.834 & 0.717 & 0.728 & 1.000 & 1.003 \\
DRNet         & 0.721 & 0.719 & 0.648 & 0.653 & 0.981 & 0.967 \\
\cmidrule(lr){1-7}
DiffPO        & \textbf{0.016}$^\star$ & \textbf{0.026}$^\star$ & \textbf{0.022}$^\star$ & \textbf{0.019}$^\star$  & \textbf{0.038}$^\star$ & \textbf{0.021}$^\star$ \\
Pretrain      & 0.238 & 0.230 & 0.183 & 0.181 & 0.333 & 0.310 \\
IWDD         &  0.229 & 0.229   & 0.195 & 0.181 &  0.320     &   0.322    \\
\midrule
\addlinespace
& \multicolumn{2}{c}{\texttt{8ff38d...}} & \multicolumn{2}{c}{\texttt{74420a...}} & \multicolumn{2}{c}{\texttt{110f6dc...}} \\
\cmidrule(lr){2-3} \cmidrule(lr){4-5} \cmidrule(lr){6-7}
& PEHE$_{\text{in}}$ & PEHE$_{\text{out}}$ & PEHE$_{\text{in}}$ & PEHE$_{\text{out}}$ & PEHE$_{\text{in}}$ & PEHE$_{\text{out}}$ \\
\midrule
Causal Forest & 12.302 & 12.298 & 13.113 & 11.653 & 0.523 & \textbf{0.525} \\
T-learner     & 12.390 & 12.318 & 19.321 & 13.813 & \textbf{0.521} & \textbf{0.525} \\
S-learner     & 9.436  & 9.382  & 8.398  & 6.172  & 0.560 & 0.563 \\
TNet          & 0.717  & 0.718  & 0.404  & 0.382  & 0.594 & 0.588 \\
TARNet        & 0.734  & 0.743  & 0.448  & 0.456  & 0.617 & 0.613 \\
OffsetNet     & 0.701  & 0.714  & 0.465  & 0.475  & 0.563 & 0.556 \\
FlexTENet     & 0.667  & 0.676  & 0.448  & 0.449  & 0.612 & 0.603 \\
DRNet         & 0.714  & 0.712  & 0.426  & 0.414  & 0.684 & 0.680 \\
\cmidrule(lr){1-7}
DiffPO  & 3590.180 & 3585.164 & 0.017 & 0.054 & 0.828$^\star$ & 0.830$^\star$ \\
Pretrain      & 0.014  & 0.014  & \textbf{0.002}$^\star$  & \textbf{0.002}$^\star$ & 1.032 & 1.023 \\
IWDD          &  \textbf{ 0.009}$^\star$  &  \textbf{ 0.009}$^\star$  & \textbf{0.002}$^\star$ & \textbf{0.002}$^\star$ &1.044 & 1.032   \\
\midrule
\addlinespace
 & \multicolumn{2}{c}{\texttt{3ebc516...}} & \multicolumn{2}{c}{\texttt{5a147c...}} & \multicolumn{2}{c}{\texttt{5ad181...}} \\
\cmidrule(lr){2-3} \cmidrule(lr){4-5} \cmidrule(lr){6-7}
& PEHE$_{\text{in}}$ & PEHE$_{\text{out}}$ & PEHE$_{\text{in}}$ & PEHE$_{\text{out}}$ & PEHE$_{\text{in}}$ & PEHE$_{\text{out}}$ \\
\midrule
Causal Forest & 14.542 & 14.531 & 33.209 & 32.968 & 42.400 & 42.450 \\
T-learner     & 14.957 & 14.743 & 35.354 & 33.761 & 42.580 & 42.509 \\
S-learner     & 10.452 & 10.354 & 14.072 & 13.650 & 32.033 & 32.030 \\
TNet          & 0.707  & 0.667  & 0.706  & 0.667  & 0.731  & 0.729 \\
TARNet        & 0.699  & 0.672  & 0.699  & 0.672  & 0.790  & 0.809 \\
OffsetNet     & 0.767  & 0.756  & 0.767  & 0.756  & 0.721  & 0.721 \\
FlexTENet     & 0.676  & 0.633  & 0.676  & 0.633  & 0.785  & 0.795 \\
DRNet         & 0.716  & 0.698  & 0.716  & 0.698  & 0.773  & 0.783\\
\cmidrule(lr){1-7}
DiffPO        & 5.037 & 4.198 &3402.876 & 3479.084 & 0.932 & 0.232 \\
Pretrain     & 0.106  & 0.084 & 0.103 & 0.069 & \textbf{0.027}$^\star$ & \textbf{0.0001}$^\star$   \\
IWDD          &    \textbf{0.052}$^\star$ & \textbf{0.052}$^\star$   & \textbf{0.034}$^\star$ & \textbf{0.034}$^\star$     & 0.057 & \textbf{0.0001}$^\star$       \\
\bottomrule
\end{tabular}
\end{table}

\begin{table}[htbp]
\centering
\small
\caption{Win rates (\%) and mean RMSE\textsubscript{±SD} for $Y(0)$, $Y(1)$ across 12 ACIC 2018 datasets. The best result across all methods is highlighted in \textbf{bold}, and the best result among the diffusion-based approaches (DiffPO, Pretrain, IWDD) is additionally marked with a $\star$.}
\label{tab:rmse-summary-acic2018-sd}
\begin{tabular}{lcccccccc}
\toprule
\textbf{Method} & \multicolumn{4}{c}{In-sample} & \multicolumn{4}{c}{Out-of-sample} \\
\cmidrule(lr){2-5} \cmidrule(lr){6-9}
& Win$_0$ & RMSE$_0$ & Win$_1$ & RMSE$_1$ & Win$_0$ & RMSE$_0$ & Win$_1$ & RMSE$_1$ \\
\midrule
T-learner     & 0 & 65\textsubscript{±85}& 8.3 & 68\textsubscript{±86} & 0 & 301\textsubscript{±782} & 8.3 & 307\textsubscript{±801} \\
S-learner     & 8.3 & 64\textsubscript{±83} & 0 & 66\textsubscript{±83} & 8.3 & 301\textsubscript{±783}  & 0 & 307\textsubscript{±801} \\
TNet          & 0 & 0.86\textsubscript{±0.29} & 0 & 0.96\textsubscript{±0.31} & 8.3 & 1.02\textsubscript{±0.19} & 0 & 1.11\textsubscript{±0.19} \\
TARNet        & 0 & 0.86\textsubscript{±0.29} & 0 & 0.96\textsubscript{±0.33} & 0 & 1.02\textsubscript{±0.19} & 0 & 1.11\textsubscript{±0.19} \\
OffsetNet     & \textbf{75} & \textbf{0.79}\textsubscript{±0.27} & \textbf{33.3} & \textbf{0.89}\textsubscript{±0.31} & 0 & 1.11\textsubscript{±0.19} & 0 & 1.22\textsubscript{±0.18} \\
FlexTENet     & 8.3 & 0.88\textsubscript{±0.16} & 16.7 & 0.99\textsubscript{±0.19} & 0 & 1.06\textsubscript{±0.19} & 0 & 1.15\textsubscript{±0.18} \\
\cmidrule(lr){1-9}
DiffPO        & 8.3$^\star$ & 474\textsubscript{±1091} & 25$^\star$ & 975\textsubscript{±2271} & 16.7 & 473\textsubscript{±1100} & 25 & 975\textsubscript{±2271} \\
Pretrain      & 0 & 0.99\textsubscript{±0.18} & 0 & 1.17\textsubscript{±0.22} & 0 & 0.90\textsubscript{±0.34} & 0 & 1.08\textsubscript{±0.39} \\
\textbf{IWDD}          & 0 & 0.96$^\star$\textsubscript{±0.17} & 16.7 & 1.04$^\star$\textsubscript{±0.24} & \textbf{66.7} & \textbf{0.88}$^\star$\textsubscript{±0.33} & \textbf{75} & \textbf{0.95}$^\star$\textsubscript{±0.40}\\
\bottomrule
\end{tabular}
\end{table}

\begin{table}[ht]
\centering
\small
\caption{Win rates and mean PEHE\textsubscript{±SD} across 12 ACIC 2018 datasets}
\label{tab:pehe-summary-acic2018-sd}
\begin{tabular}{lcccc}
\toprule
 & \multicolumn{2}{c}{In-sample} & \multicolumn{2}{c}{Out-of-sample} \\
\cmidrule(lr){2-3} \cmidrule(lr){4-5}
& Win rate (\%) & Mean PEHE & Win rate (\%) & Mean PEHE \\
\midrule
Causal Forest & 0\% & 13.452\textsubscript{±12.552} & 8\% & 18.190\textsubscript{±19.335} \\
T-learner     & 8\% & 15.175\textsubscript{±13.270} & 8\% & 18.724\textsubscript{±19.280} \\
S-learner     & 0\% & 8.613\textsubscript{±8.353} & 0\% & 13.619\textsubscript{±18.311} \\
TNet          & 0\%    & 0.600 \textsubscript{±0.228} & 0\% & 0.574\textsubscript{±0.230}  \\
TARNet        & 0\% & 0.618\textsubscript{±0.237} & 0\% & 0.619\textsubscript{±0.234} \\
OffsetNet     & 0\% & 0.606\textsubscript{±0.256} & 0\% & 0.606\textsubscript{±0.250} \\
FlexTENet     & 0\% & 0.652\textsubscript{±0.192} & 0\% & 0.646\textsubscript{±0.193} \\
DRNet         & 0\% & 0.607\textsubscript{±0.226} & 0\% & 0.606\textsubscript{±0.217} \\
\cmidrule(lr){1-5}
DiffPO        & 33\% & 583.4\textsubscript{±1361.3} & 33\% & 589.1\textsubscript{±1374.9} \\
Pretrain      & 25\% & 0.318\textsubscript{±0.532} & 33\% & 0.306\textsubscript{±0.535} \\
\textbf{IWDD} & \textbf{50\%} & \textbf{0.308}\textsubscript{±0.538} & \textbf{58\%} & \textbf{0.301}\textsubscript{±0.539} \\
\bottomrule
\end{tabular}
\begin{tablenotes}
\small
\item \textit{Note:} Ties are counted for both methods when calculating win rates.
\end{tablenotes}
\end{table}

\clearpage
\subsection{IHDP}
\label{app:ihdp_exp_results}

Tables~\ref{tab:ihdp_pehe} and~\ref{tab:ihdp_rmse} summarize the results on the IHDP dataset. IWDD achieved the best out-of-sample PEHE performance. However, since the IHDP dataset contains only 747 samples, several baselines outperformed IWDD on RMSE. This is likely because IWDD has a large number of parameters to optimize, and IHDP is a relatively small dataset.

\begin{table}[htbp]
\centering
\caption{IHDP1: RMSE}
\label{tab:ihdp_rmse}
\begin{tabular}{lcccc}
\hline
\multirow{2}{*}{\textbf{Algorithm}} & \multicolumn{4}{c}{\textbf{RMSE}} \\
\cmidrule(lr){2-5}
& $Y(0)$ (in) & $Y(0)$ (out) & $Y(1)$ (in) & $Y(1)$ (out) \\
\hline
T-learner & 2.862 & 3.287 & \textbf{0.355} & \textbf{0.357} \\
S-learner & 2.976 & 3.403 & 1.393 & 1.334 \\
TNet & 0.716 & 1.415 & 0.995 & 0.982 \\
TNet\_reg & 0.825 & 1.644 & 1.052 & 1.000 \\
TARNet & 0.804 & 1.639 & 1.028 & 0.957 \\
TARNet\_reg & 0.816 & 1.650 & 0.959 & 0.890 \\
OffsetNet & 0.913 & 1.704 & 1.667 & 1.376 \\
FlexTENet & 0.757 & 1.550 & 0.985 & 0.865 \\
FlexTENet\_noortho & 0.794 & 1.631 & 1.109 & 1.003 \\
Pretrain & \textbf{0.527} & \textbf{0.533} & 1.608 & 1.630 \\
IWDD & 0.545 & 0.597 & 1.789 & 1.902 \\
\hline
\end{tabular}
\end{table}

\begin{table}[htbp]
\centering
\caption{PEHE (in-sample and out-of-sample) on IHDP1}
\label{tab:ihdp_pehe}
\begin{tabular}{lcc}
\hline
\textbf{Algorithm} & PEHE$_\text{in}$ & PEHE$_\text{out}$ \\
\hline
Causal Forest (CF)   & 4.072 & 4.265 \\
T-learner            & 2.690 & 3.102 \\
S-learner            & 3.695 & 3.966 \\
TNet                 & 1.255 & 1.886 \\
TNet\_reg            & 1.365 & 2.070 \\
TARNet               & 1.336 & 2.028 \\
TARNet\_reg          & 1.290 & 2.010 \\
OffsetNet            & 1.771 & 2.320 \\
FlexTENet            & \textbf{1.142} & 1.786 \\
FlexTENet\_noortho   & 1.375 & 2.064 \\
DRNet                & 1.276 & 1.933 \\
DRNet\_TAR           & 1.277 & 1.741 \\
GANITE               & 1.923 & 2.433 \\
Pretrain             & 1.681 & 1.658 \\
IWDD                 & 1.698 & \textbf{1.655} \\
\hline
\end{tabular}
\end{table}

\end{document}